%% file: main.tex
\documentclass{ecai}
\usepackage{graphicx}
\usepackage{latexsym}

\usepackage{times}
\usepackage{soul}
\usepackage{url}
\usepackage[utf8]{inputenc}
\usepackage{booktabs}
\usepackage{algorithm}
\usepackage{algorithmic}
\usepackage[switch]{lineno}

\usepackage{balance}
\usepackage{subfigure}

\usepackage{enumerate}
\usepackage[inline]{enumitem}
    \setlist[enumerate]{label = \emph{\roman*})}
\usepackage[most]{tcolorbox}
\usepackage{amsfonts}
\usepackage[hidelinks]{hyperref}
\usepackage[capitalize,noabbrev,nameinlink]{cleveref}
\usepackage{fontawesome5}

\definecolor{tab:blue}{HTML}{1F77B4}
\definecolor{tab:orange}{HTML}{FF7F0E}
\definecolor{tab:green}{HTML}{2CA02C}
\definecolor{tab:red}{HTML}{D62728}
\definecolor{tab:purple}{HTML}{9467BD}
\definecolor{tab:brown}{HTML}{8C564B}
\definecolor{tab:pink}{HTML}{E377C2}
\definecolor{tab:gray}{HTML}{7F7F7F}
\definecolor{tab:olive}{HTML}{BCBD22}
\definecolor{tab:cyan}{HTML}{17BECF}

\newtheorem{lemma}{Lemma}
\newtheorem{assumption}{Assumption}

\DeclareMathOperator{\E}{\mathbb{E}}
\DeclareMathOperator*{\softplus}{softplus}
\definecolor{algorithmiccommentcolor}{gray}{0.5}

\newtcolorbox{leftvrule}[1][]{colback=white,
    boxrule=0pt, boxsep=0pt, breakable, enhanced jigsaw,
    borderline west={1.5pt}{0pt}{black},
before skip=10pt,after skip=10pt,
#1}

\newcommand{\minH}{{\overline{\mathcal{H}}}}
\newcommand{\targ}{{\ensuremath{\odot}}}
\newcommand{\src}{{\ensuremath{\diamond}}}
\newcommand{\pib}{{\ensuremath{\pi_b}}}

\newcommand{\ours}{SaGui}

\graphicspath{{Figures/}}

\usepackage{xcolor}
\begin{document}

\begin{frontmatter}

\title{Reinforcement Learning by Guided Safe Exploration}

\author[A]{\fnms{Qisong}~\snm{Yang}\thanks{Equal contribution.}}
\author[B]{\fnms{Thiago}~D.~\snm{Simão}\thanks{}}
\author[B]{\fnms{Nils}~\snm{Jansen}}
\author[A]{\fnms{Simon}~H.~\snm{Tindemans}}
\author[A]{\fnms{Matthijs}~T.~J.~\snm{Spaan}}

\address[A]{Delft University of Technology -- The Netherlands}
\address[B]{Radboud University, Nijmegen -- The Netherlands}

\begin{abstract}
Safety is critical to broadening the application of reinforcement learning (RL).
Often, we train RL agents in a controlled environment, such as a laboratory, before deploying them in the real world.
However, the real-world target task might be unknown prior to deployment.
Reward-free RL trains an agent without the reward to adapt quickly once the reward is revealed.
We consider the \emph{constrained} reward-free setting, where an agent (the guide) learns to explore safely without the reward signal.
This agent is trained in a controlled environment, which allows unsafe interactions and still provides the safety signal.
After the target task is revealed, safety violations are not allowed anymore.
Thus, the guide is leveraged to compose a safe behaviour policy.
Drawing from transfer learning, we also regularize a target policy (the student) towards the guide while the student is unreliable
and gradually eliminate the influence of the guide as training progresses.
The empirical analysis shows that this method can achieve safe transfer learning and helps the student solve the target task faster.
\end{abstract}

\end{frontmatter}

\input{md}

\ack
We thank the reviewers for their insightful comments.
This work has been partially funded by the ERC Starting Grant 101077178 (DEUCE) and the NWO grant NWA.1160.18.238 (PrimaVera). 
Qisong Yang is supported by Xidian University.

\bibliography{main}

\newpage
\onecolumn
\appendix

\input{app}

\end{document}

%% file: md.tex
\section{Introduction}

Despite the numerous achievements of reinforcement learning~(RL)~\cite{sutton2018reinforcement,mnih2015human}, 
safety concerns still prevent the wide adoption of RL~\cite{Dulac-Arnold2021}.
The lack of knowledge about the environment forces standard agents to rely on trial-and-error strategies.
However, this approach is incompatible with safety-critical scenarios \cite{Garcia2015survey}.
For instance, recommender systems should not suggest extremist content~\cite{di2022recommender}.
Constrained Markov decision processes (CMDP) \cite{altman1999cmdp}
express such safety constraints with a cost signal indicating unsafe interactions.
Such costs are decoupled from the rewards to facilitate the learning of safe behaviours.

Developments in safe RL have allowed us to learn safe policies in CMDPs.
For instance, SAC-Lagrangian~\cite{ha2020learning} combines the Soft Actor-Critic (SAC)~\cite{haarnoja2018soft1,haarnoja2018soft2} algorithm with Lagrangian methods to learn a safe policy in an off-policy way.
This algorithm solves high-dimensional problems with a sample complexity lower than on-policy algorithms.
Unfortunately, it only finds a safe policy at the end of the training process and may be unsafe while learning.
In terms of safety, we consider episode-wise constraints instead of step-wise constraints, so a few unsafe actions are allowed in an episode. 

Some knowledge about the safety dynamics can ensure safety during learning.
One can pre-compute unsafe behaviour and mask unsafe actions using a so-called shield \cite{Alshiekh2018,Jansen2020,shield-pomdp}, or start from an initially safe baseline policy and gradually improve its performance while remaining safe \cite{achiam2017cpo,Tessler2019,Yang2020}.
However, these approaches may necessitate numerous interactions with the environment before they find an adequate policy~\cite{Zanger2021}.
Moreover, reusing a pre-trained policy can have a detrimental effect, since the agent encounters a new trajectory distribution as the policy changes~\cite{Igl2021}.
Therefore, we investigate \emph{how to efficiently solve a task without violating the safety constraints}.

We make two key observations.
First, RL agents often learn in a controlled environment, such as a laboratory or a simulator, before being deployed in the real world~\cite{Garcia2015survey}.
Second, an agent can often benefit from expert guidance instead of solely relying on trial and error~\cite{peng2022safe}.
For instance, in autonomous driving, the driver agent can quickly learn by mimicking an expert's behaviour to handle dangerous situations. 
Such a process is referred to as \textit{policy distillation}. 
Furthermore, under expert guidance, the agent can safely explore before taking dangerous actions.

Transfer learning~\cite{Taylor2009transfer} investigates how to improve the learning of a target task with some knowledge from a source task.
In these settings, the source task may provide only partial knowledge of the target task.
We adopt a transfer learning framework and refer to (i) the controlled environment as the \textit{source task}~($\src$) and (ii) the real world as the \textit{target task}~($\targ$).
In our setting, the controlled environment provides only the cost signals related to safety but not the reward signals of the target task in the real world.
The central problem is then to avoid safety violations after the target task has been revealed.

\textbf{Our approach.}
We show how to transfer knowledge encoded by a policy to enhance safety.
Here, we refer to the policy that has been learned in the source task as the \textit{safe guide} (\ours{}, \cref{fig:sagui}).
The intuition is that, in the real world, the agent is guided to accomplish the target task in a safe manner.
We propose to transfer \ours{} from the source task to the target task.
Our approach has three central steps:
\begin{enumerate*}
    \item train the \ours{} policy and \emph{transfer} it to the target task;
    \item \emph{distill} the guide's policy into a \emph{student policy} which is dedicated to the target task, and
    \item \emph{compose} a behaviour policy that balances safe exploration (using the guide) and exploitation (using the student).
\end{enumerate*}

\begin{figure}[tbp]
\centering 
    \resizebox{.84\columnwidth}{!}{\def\svgwidth{240pt}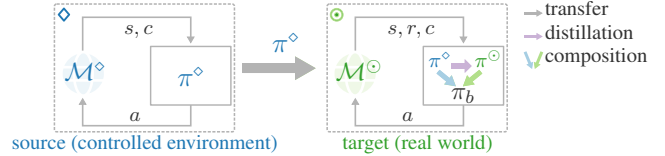}
\caption{Transferring the Safe Guide (SaGui) policy $\pi^\src$ from the source task~($\src$) to the target task~($\targ$) with three steps.
    } 
\label{fig:sagui} 
\end{figure}

As we train the guide in a reward-free constrained RL setting \cite{DBLP:conf/icml/MiryoosefiJ22}, the agent only observes the costs related to safety, and it does not access reward signals.
This task-agnostic approach allows us to train a guide independently of the reward of the target task, so this guide can be useful for different reward functions.
Furthermore, we assume the source task preserves the dynamics related to safety, which allows us to train a guide that can act safely when transferred to the target task.
Inspired by advances in robotics where an agent is trained under strict supervision, we assume the source task is a simulated/controlled environment~\cite{Schuitema2010,Xie2019}.
Therefore, safety is not required while training the \ours{} policy.
Once the target task is revealed, \ours{} safely collects the initial trajectories in the target environment and the student starts learning based on these trajectories.
To ensure that the new policy quickly learns how to act safely, we also employ a policy distillation method, encouraging the student to imitate the guide.

\textbf{Contributions.}
Our main contributions are: we
\begin{enumerate*}
    \item formalize transfer learning for RL from a safety perspective;
    \item propose to guide learning using a task-agnostic agent with exploration benefits;
    \item show how to adaptively regularize the student policy to the guide policy based on the student's safety;
    \item investigate when to sample from the student or from the guide to ensure safe behaviour in the target environment and fast convergence of the student policy; and
    \item demonstrate empirically that, compared to learning from scratch and adapting a pre-trained policy, our method can solve the target task faster without violating the safety constraints in the target task.
\end{enumerate*}

\section{Related Work}
\label{sec:related-work}

Safe RL has multiple facets \cite{Garcia2015survey}, ranging from alternative optimization criteria \cite{Yang2021,chow2017risk} to safe exploration based on some prior knowledge  \cite{achiam2017cpo,Alshiekh2018,Jansen2020,Sui2015,yang2021accelerating,Simao2021alwayssafe}.
We review methods to train the guide and solve new tasks using a pre-trained policy.

Multiple algorithms have been proposed for generalizing policies from reward-free RL for better performance in target tasks \cite{Zhang2020,gimelfarb2021risk,srinivasan2020learning}.
However, only a few works have considered reward-free RL with constraints~\cite{DBLP:conf/icml/MiryoosefiJ22,Savas2018}.
They focus on tabular and linear settings while we consider general function approximation algorithms.

Work in transfer learning has leveraged meta-RL \cite{finn2017model} for safe adaptation \cite{Grbic2020,luo2021mesa,Lew2020safe}.
Our work is also related to curriculum learning  \cite{bengio2009curriculum,turchetta2020safe,marzari2021curriculum}.
We first train an agent to be safe and later solve a target task.
However, our approach focuses on safe exploration and is able to transfer to tasks with different reward functions, so the guide's training is ignored. 

Our work resembles certain safe transfer-RL frameworks~\cite{karimpanal2020learning,yang2021accelerating}, which also leverage prior knowledge to aid learning in a target task. However, the SaGui framework differs from them in terms of safety definition, knowledge acquisition in the source task, or knowledge usage in the target task.
Our prior knowledge is more effective for various downstream tasks, and SaGui is the only framework that is safe while learning in the target task.

\section{Background}

We formalize the safe RL problem and describe typical approaches.

\subsection{Constrained Markov Decision Processes}

We consider tasks formulated by constrained Markov decision processes (CMDPs) \cite{altman1999cmdp,Borkar2005}. A CMDP is defined as a tuple $\mathcal{M} = \langle \mathcal{S}, \mathcal{A}, \mathcal{P}, r, c, d, \gamma\rangle$: 
a state space $\mathcal{S}$,
an action space $\mathcal{A}$,
a probabilistic transition function $\mathcal{P}\colon \mathcal{S} \times \mathcal{A} \mapsto \mathit{Dist}(\mathcal{S})$,
a reward function $r\colon \mathcal{S} \times \mathcal{A} \mapsto [r_{min}, r_{max}]$,
a cost function $c\colon \mathcal{S} \times \mathcal{A} \mapsto [c_{min}, c_{max}]$,
a safety threshold~$d \in \mathbb{R}^+$, and
a discount factor $\gamma \in [0,1)$.
We also consider an initial state distribution $\iota\colon \mathcal{S} \mapsto [0,1]$.
In a \emph{constrained RL} problem, an agent interacts with a CMDP without knowledge about the transition, reward, and cost functions, generating a trajectory $\tau = \langle (s_0,a_0,r_0,c_0,s_0'),(s_1,a_1,r_1,c_1,s_1'),\cdots \rangle$.
A trajectory starts from $s_0\sim \iota(\cdot)$. Then, at each timestep $t$ the agent is in a state $s_t \in \mathcal{S}$, and takes an action $a_t \in \mathcal{A}$. It subsequently gets a reward $r_t = r(s_t, a_t)$, a cost $c_t = c(s_t, a_t)$, and steps into a new state $s_{t}' \sim \mathcal{P}(\cdot\mid s_t,a_t)$.
This process repeats starting from $s_{t+1}=s_t'$ until a terminal condition is met and a new trajectory starts.
The goal is to learn a policy $\pi$ that maximizes the expected discounted return such that the expected discounted cost-return remains below $d$:
\newcommand{\sampletraj}{\ensuremath{\rho_{\pi}}}
\begin{equation}
\begin{aligned}
    \max_\pi & \E_{\sampletraj} \left[ \sum_{t=0}^{\infty} \gamma^t r_t \right] 
\quad {\rm s.t.} \quad  \E_{\sampletraj} \left[\sum_{t=0}^{\infty} \gamma^t c_t\right] \leq d,
\end{aligned}
\label{eq:CMDP}
\end{equation}
where $\rho_{\pi}$ indicates the trajectory distribution induced by $s_0 \sim \iota(\cdot)$, $a_t \sim \pi(\cdot\mid s_t)$, and $s_{t+1} \sim \mathcal{P}(\cdot\mid s_t,a_t)$.
We define the discounted \textit{return} starting from $s,a$ and following $\pi$ as
$
    Q_{\pi}^{r}(s,a) = \E_{\sampletraj} \left[\sum^{\infty}_{t=0} \gamma^t  r_t \middle| s_0 = s, a_0 = a\right],
$
and, similarly, the discounted \textit{cost-return}  $Q_{\pi}^{c}(s,a)$.

From the safe RL perspective, if a policy has an expected cost-return lower than the safety-threshold~$d$, then this policy is considered safe.
Therefore, the objective of a safe RL agent is to find a policy, among the safe policies, that has the highest expected return.

\subsection{Maximum Entropy Reinforcement Learning}
\label{sec:MERL}

A common strategy to improve the exploration and robustness of RL agents is to favour policies that induce diverse behaviours \cite{ziebart2010modeling,eysenbach2021maximum}.
We can incorporate it in the safe RL main objective by augmenting the problem with a term that aims to maximize the policy entropy~\cite{Haarnoja2017}:
\begin{equation}
    \max_{\pi} \E_{\sampletraj} \left[ \sum_{t=0}^{\infty} \gamma^t \left(r_t {+} \alpha \mathcal{H}(\pi(s_t)) \right) \right]
\textrm{ s.t.} \E_{\sampletraj} \left[ \sum_{t=0}^{\infty} \gamma^t c_t \right] {\leq} d,
\label{eq:SafeMERL}
\end{equation}
where $\mathcal{H}(\cdot)$ is the entropy of a distribution over a random variable, and $\alpha$ is the entropy weight.
In general, this objective encourages the agent to use maximally stochastic policies.
Alternatively, we can encourage the policy to have at least a minimum entropy $\minH$ \cite{haarnoja2018soft2} by adding the following constraint to \eqref{eq:CMDP}:
\begin{equation}
    \E_{\sampletraj} \left[- \log(\pi(a_t\mid s_t))  \right] \geq \minH, \quad \forall t,
    \label{eq:constraint_minimum_entropy}
\end{equation}
where $\minH$ is the given entropy threshold to ensure a minimum degree of randomness.
This approach allows the policy to converge to a more deterministic behaviour than \eqref{eq:SafeMERL}.
Besides, it only requires the system's designer to define $\minH$ and it lets the RL agent automatically find a trade-off between the policy's entropy and rewards.
Therefore, $\alpha$ becomes an intrinsic parameter of the RL algorithm.

\begin{figure*}[tbp]
    \centering
    \hfill
    \subfigure[Unsafe transfer]{\label{subfig:safe_transfer_metrics}
        \resizebox{0.285\textwidth}{!}{\def\svgwidth{200pt}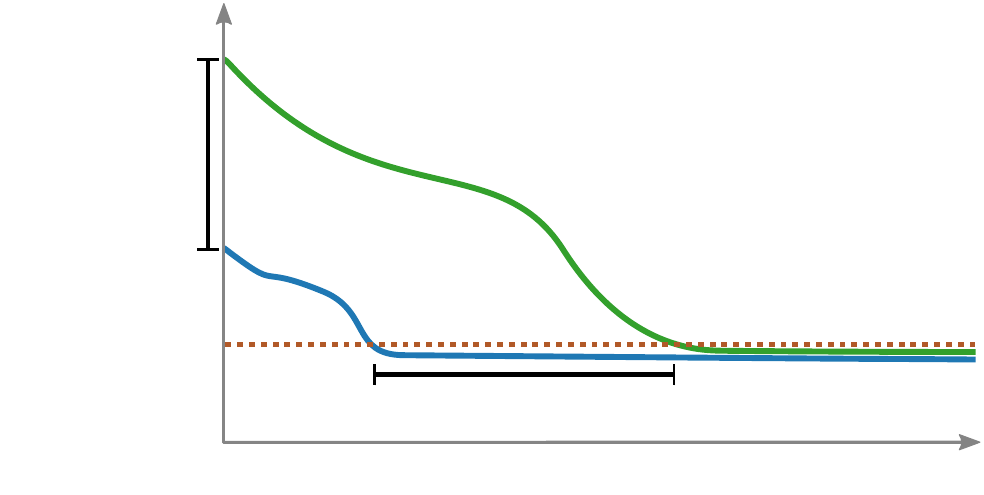}
    }\hfill
    \subfigure[Fully safe transfer.]{\label{subfig:safe_transfer}
        \resizebox{0.285\textwidth}{!}{\def\svgwidth{200pt}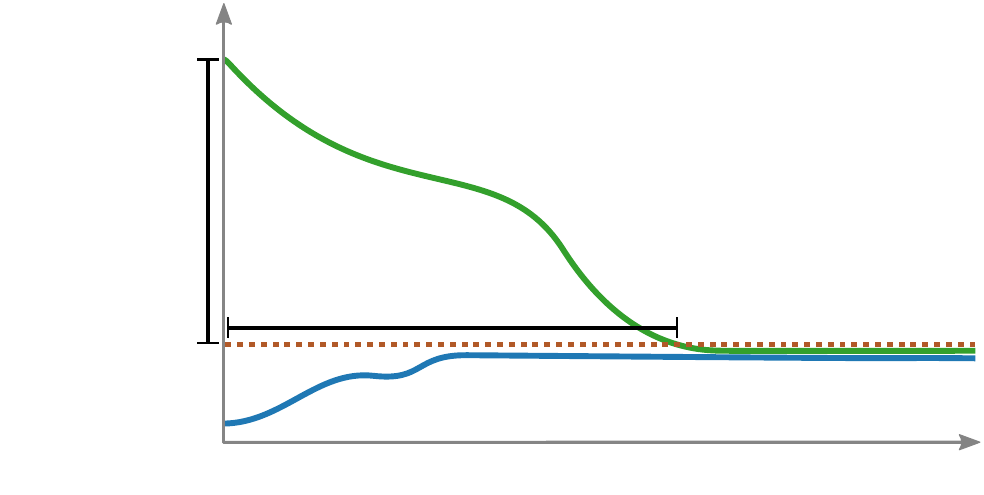}
    }\hfill
    \subfigure[Return transfer.]{\label{subfig:return_transfer}
        \resizebox{0.285\textwidth}{!}{\def\svgwidth{200pt}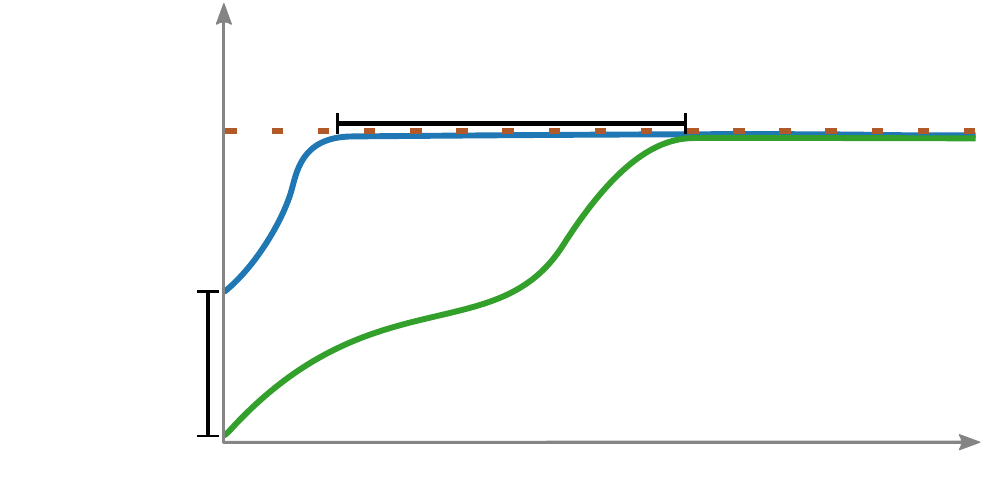}
    }\hfill
    \caption{Transfer metrics for safe reinforcement learning. Usually, we consider \textit{safety jump-start} and \textit{$\Delta$~time to safety}. If we can develop agents that learn without violating the safety requirements, we can also consider \emph{return jump-start} and \textit{$\Delta$~time to optimum}.}
\end{figure*}

The maximum entropy RL with safety constraint \eqref{eq:SafeMERL} can be solved by the SAC-Lagrangian (SAC-$\lambda$)~\cite{ha2020learning} method.
SAC-$\lambda$ is a SAC-based method that has two critics and uses an adaptive entropy weight $\alpha$ (parameterized by $\theta_\alpha$) and an adaptive safety weight $\beta$ (parameterized by $\theta_\beta$) to manage a trade-off among exploration, reward, and safety. 
The reward critic estimates the expected return $Q^r$ (parameterized by $\theta_R$), possibly with an entropy bonus to promote exploration, while the safety critic estimates the cost-return~$Q^c$ (parameterized by $\theta_C$) to encourage safety. The policy~$\pi$ is parameterized by $\theta_\pi$.
\cref{app:saclag} provides a detailed description of how to learn each component, including the losses.
Throughout the paper, we represent learning rates with $\eta$, replay buffers with $\mathcal{D}$, and losses with $J$.
We only update $\alpha$ when a desirable $\minH$ is given, so $\alpha$ is fixed whenever we use the formulation \eqref{eq:SafeMERL}.

\section{Safe and Efficient Exploration}

Naturally, to train RL agents without violating the safety constraints, some prior knowledge is required \cite{Sui2015}.
Often, a safe initial policy collects the initial trajectories \cite{achiam2017cpo,Tessler2019,Yang2020}.
However, these approaches largely neglect how this policy is computed or what makes it effective.
Therefore, we consider the problem of how to obtain an initial policy that can safely expedite learning in the target task.
Next, we formalize the problem and provide an overview of our approach.

\subsection{Problem Setting}

We formalize our problem setting using the transfer learning (TL) framework.
In general, TL allows RL agents to use expertise from \textit{source} tasks to speed up the learning process on a \textit{target} task~\cite{Taylor2009transfer,zhu2020transfer}.
The source tasks $\{\mathcal{M}^\src\}$ should provide some knowledge~$\mathcal{K}^\src$ to an agent learning in the target task $\mathcal{M}^\targ$, such that, by leveraging $\mathcal{K}^\src$, the agent learns the target task $\mathcal{M}^\targ$ faster.

As we are particularly interested in the safety properties of the transfer, we consider a reward-free source task, which only provides information about the safety dynamics.
Moreover, we use a policy to encode the knowledge transferred.
Formally, given a source task $\mathcal{M}^\src = \langle \mathcal{S}^\src, \mathcal{A}^\src, \mathcal{P}^\src, \emptyset, c^\src, d^\src, \iota^\src, \gamma \rangle$, we compute the policy $\pi^\src$ in the absence of a reward signal.
This provides knowledge~$\mathcal{K}^\src = \{\pi^\src\}$ to help solving the target task $\mathcal{M}^\targ = \langle \mathcal{S}^\targ, \mathcal{A}^\targ, \mathcal{P}^\targ, r^\targ, c^\targ, d^\targ, \iota^\targ, \gamma \rangle$.

To apply the source policy $\pi^\src$ in the target task $\mathcal{S}^\targ$, we have a mapping from the source state space to the target state space $\Xi: \mathcal{S}^\targ \rightarrow \mathcal{S}^\src$.
Then, we can define a target policy $\pi^{\src \rightarrow \targ}$ as follows:
$
    \pi^{\src \rightarrow \targ}(s)
    =
    \pi^\src(\Xi(s)).
$
Furthermore, we assume the source task $\mathcal{M}^\src$ and target task $\mathcal{M}^\targ$ share the same action space.
\cref{app:state_abstraction} describes how to obtain the source task based on $\Xi$ and the target task.

\begin{assumption}\label{a:shared_action}
$\mathcal{A}^\src=\mathcal{A}^\targ = \mathcal{A}$.
\end{assumption}

To enable the knowledge transferable between tasks, having the same action spaces ensures that the policy learned in the source task is directly applicable to the target task.

\subsection{Transfer Metrics}
\label{sec: transfer-metrics}

To evaluate a safe transfer RL algorithm, \cref{subfig:safe_transfer_metrics} presents a schematic of metrics related to safety (inspired by transfer in RL~\cite{Taylor2009transfer}):
\textit{safety jump-start} indicates how much closer to the safety threshold the expected cost-return of an agent learning with the source knowledge is compared to the expected cost-return of an agent learning from scratch in the first episodes, and 
\textit{$\Delta$~time to safety} is the difference in the number of interactions required to become safe.

Notice that a trained agent might start with an expected cost-return lower than the safety threshold, for instance, when the safety threshold in the source task is lower than in the target task (\cref{subfig:safe_transfer}).
In this case, \textit{safety jump-start} would be the difference between the safety threshold and the cost-return of an agent learning from scratch.
Similarly, the \textit{$\Delta$~time to safety} would be the number of interactions an agent learning from scratch needs to become safe.

In the case of two methods that can solve the target task without violating the safety constraints, we can also consider the usual metrics of transfer learning with respect to the reward \cite{Taylor2009transfer}.
For instance, \cref{subfig:return_transfer} shows the initial improvement in terms of performance which we call \emph{return jump-start}, and the time necessary to reach an optimum performance, which we call the \textit{$\Delta$~time to optimum}.

\paragraph{Problem statement.}
\label{sec:problem-setting}
        We aim to maximize the \textit{safety jump-start} (potentially preventing safety violations in the target task) and to reduce the \textit{time to optimum }(improving exploration) when transferring a policy~$\pi^\src$ from a source task~$\mathcal{M}^\src$ to a target task~$\mathcal{M}^\targ$.

\subsection{Method Overview}

\label{method-overview}

Recall that for our transfer setting, we consider a single source task that only provides the safety signals, which we use to train the guide.
Without the reward signal, the guide aims to explore the world safely and efficiently.
We are interested in using the guide's safe exploration capabilities to train the student on the target task without violating the safety constraints.
Notably,
\textit{i}) the guide and the student are trained separately;
\textit{ii}) the guide is only trained once and can support the training of different students; and
\textit{iii}) the guide only has access to safety information and no knowledge about the student's task.

To ensure the source policy is safe when deployed in the target task, we assume that the source task has a safety threshold lower than or equal to the target task, and $\Xi$ is a state abstraction that preserves the safety dynamics, as formalized next.

\begin{assumption}\label{a:cost}
The safety threshold of the target task upper bounds the safety threshold of the source task:
$d^\src \leq d^\targ$.
\end{assumption}

\begin{assumption}\label{a:abstraction}
$\Xi$ is a $Q_{\pi}^{c}$-irrelevance abstraction~\cite{Li2006}, therefore
\[
\Xi(s) {=} \Xi(s')
        \Rightarrow Q_{\pi^\targ}^{c}(s, a) = Q_{\pi^\targ}^{c}(s', a), 
    \forall s, s' \in \mathcal{S}^\targ, a \in \mathcal{A}, \pi^\targ.
\]
\end{assumption}
\noindent
Now, we can connect the expected cost-return of a policy on the source task to the expected cost-return on the target task.
\begin{lemma}\label{lem:equal_cost}
Given \cref{a:shared_action} and \cref{a:abstraction}, we have
\[
    Q_{\pi^\src}^{c,\src}(\Xi(s),a) 
    =
    Q_{\pi^{\src \rightarrow \targ}}^{c,\targ}(s,a)
    \quad \forall s \in \mathcal{S}^\targ, a \in \mathcal{A}, \pi^\src.
\]
That is, the expected cost of a source policy is the same in the source task and in the target task.
\end{lemma}
\begin{proof}
\cref{sec:proof_lemma} provides the proof.
\end{proof}

\begin{theorem}
If $~\Xi$ is a $Q_{\pi}^{c}$-irrelevant state abstraction, then any policy that is safe on the source task $\mathcal{M}^\src$ is also safe when deployed on the target task $\mathcal{M}^\targ$.
\end{theorem}
\begin{proof}~\\
$
Q_{\pi^{\src \rightarrow \targ}}^{c,\targ}(s,a)
    \overset{\text{\cref{lem:equal_cost}}}{=} Q_{\pi^{\src}}^{c,\src}(\Xi(s),a) \label{eq:step1} 
    \overset{\text{Premise}}{\leq} d^\src 
    \overset{\text{\cref{a:cost}}}{\leq} d^\targ.
$\qedhere
\end{proof}

It is important to note, however, that the reward function $r^\targ$ in the target task may be unrelated to the state space of the source task~$\mathcal{S}^\src$. 
Therefore, although a policy that is safe on the source task is also safe on the target task, the behaviour required to accomplish the target task may not be defined on the source task.
Consider, for instance, an agent with access to its position and the position of a threat.
In each target task, the agent might need to visit a different goal position, which is not defined in the source task.
Then, a safe policy may be conditioned only on the positions of the agent and the threat, but to achieve the target, the agent must consider the goal position.
This highlights the need to compute a policy dedicated to the target task.

\section{Guided Safe Exploration}

In this section, we consider how to train the \textit{safe guide}  (\ours{}) policy.
Then, we describe how the student learns to imitate the \ours{} policy after the task is revealed while learning to complete the target task.
Finally, we investigate how to prevent safety violations while the student has not yet learned how to act safely.

\subsection{Training the Safe Guide}
\label{app:train_sagui}

Since the source task does not provide information regarding the reward of the target task, we adopt a reward-free exploration approach to train the guide.
To efficiently explore the world, we first consider maximizing the policy entropy under safety constraints. Then, we can solve the problem defined in Equation~\ref{eq:SafeMERL} with $r(s,a) = 0: \forall s \in \mathcal{S}, a \in \mathcal{A}$ to get a guide \textsc{MaxEnt}.
However, although \textsc{MaxEnt} tends to have diverse behaviours, that does not imply efficient exploration of the environment.
Especially for continuous state and action spaces, it is possible that a policy provides limited exploration even if it has high entropy.

To enhance the exploration of the guide, we adopt an auxiliary reward that motivates the agent to visit novel states.
To measure the novelty, we first define the metric space $(\mathcal{S}^\ddagger, \delta)$, where $\mathcal{S}^\ddagger$ is an abstracted state space and $\delta: \mathcal{S}^\ddagger \times \mathcal{S}^\ddagger \rightarrow [0,\infty)$ is a distance function:
\begin{align*}
        \delta(s,s') = 0 &\Leftrightarrow s=s',\\
        \delta(s,s') &= \delta(s',s), \text{and}\\
        \delta(s',s'') &\leq \delta(s,s') + \delta(s,s''), & \forall s,s',s'' \in \mathcal{S}.
\end{align*}
Note that $\mathcal{S}^\ddagger$ may not be the original state space $\mathcal{S}$. Especially when $\mathcal{S}$ is high-dimensional, $\mathcal{S}^\ddagger$ can be some selected dimensions from $\mathcal{S}$, or a latent space from representation learning.
Next, we define the auxiliary rewards as the expected distance between the current state and the successor state:
\begin{align}
        r^\delta_t(s_t, a_t) = \E_{s_{t+1}\sim \mathcal{P}(\cdot \mid s_t, a_t)}\left[\delta(f^\ddagger(s_t),f^\ddagger(s_{t+1}))\right],
    \label{eq:DefAuxRew}
\end{align}
where we may apply a potential abstraction $f^\ddagger: \mathcal{S} \rightarrow \mathcal{S}^\ddagger$. So, we train the \textit{guide} agent by solving the constraint optimization problem \eqref{eq:SafeMERL} based on the auxiliary reward $r^\delta$.
Then, we can use SAC-$\lambda$ directly employed to solve \eqref{eq:SafeMERL}, as \cref{alg:safe-explorer} shows (\cref{app:saclag}). In future research, we will also investigate different distance functions to understand their effects on exploration. 

This auxiliary reward does not explicitly promote exploration, but we find that increasing the step size and policy entropy significantly improves exploration in practice.
Overall, our experiment with the auxiliary reward aimed to evaluate the impact of the exploration of the guide on how safely and quickly the student learns.

We could also consider more sophisticated reward-free exploration strategies such as maximizing the entropy of the state occupancy distribution~\cite{seo2021state,svidchenko2021maximum,hazan2019provably}.
We leave this as future work and focus on using the guide to improve how the student learns.

\input{algorithms/guided_safe_exploration}
\subsection{Policy Distillation From the Safe Guide}

When the agent is trained for a certain task, it is difficult to generalize when faced with a new task~\cite{Igl2021}.
Similarly, it is not trivial to adjust the guide's policy that was trained to explore the environment to perform the target task.
Therefore, we train a new policy, referred as the student, dedicated to the target task.

We can leverage the \textit{guide} to quickly learn how to act safely.
Through the mapping function $\Xi$, the transferred policy can be used by most constrained RL algorithms to regularize the student policy~$\pi^\targ$ towards the guide policy $\pi^\src$ using KL divergences, as shown in \cref{fig:RewardShaping}.
So, with $\pi^\src$ fixed, we have an augmented reward function
$
    r'_t = r^\targ_t + \omega r^{\textrm{KL}}_t  + \alpha r^{\mathcal{H}}_t,
$
where $r^{\textrm{KL}}_t = \log \frac{\pi^\src(a_t\mid \Xi(s_t))}{\pi^\targ(a_t\mid s_t)}$ and $r^{\mathcal{H}}_t = -\log \pi^\targ(a_t \mid s_t)$.
The weights $\omega$ and $\alpha$ indicate the strengths of the KL and entropy regularization (respectively).
\cref{app:regularized_reward} shows that setting  $r^\src_t = \log \pi^\src(a_t \mid \Xi(s_t))$ we obtain 
$\omega r^{\textrm{KL}} {+} \alpha r^\mathcal{H} = \omega r^\src + (\omega + \alpha) r^\mathcal{H}$.
Therefore, we can define the student's objective:
\begin{equation}\label{eq:newConOptRewSha}
    \begin{aligned}
    \mathop{\max} \limits_{\pi^\targ} \E_{\tau \sim \rho_{\pi^\targ}}  \sum^{\infty}_{t=0} \gamma^t \Big[ r^\targ_t + \omega r^\src
_t  + (\alpha + \omega) r^{\mathcal{H}}
_t \Big].
    \end{aligned}
\end{equation}

\begin{figure}[tbp] 
\centering 
\resizebox{\columnwidth}{!}
{\def\svgwidth{350pt}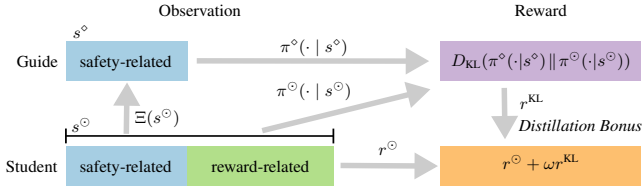}
\caption{Overview of the policy distillation.
Through the mapping function~$\Xi$, the transferred policy can be used to regularize the student policy $\pi^\targ$ towards the guide policy $\pi^\src$.
}
\label{fig:RewardShaping} 
\end{figure}

To find an appropriate $\omega$, our goal is to follow the guide more for safer exploration if the student's policy is unsafe, but eliminate the influence from the guide and focus more on the performance if the student's policy is safe.
Therefore, we propose to set $\omega = \beta$ to determine the strength of the KL regularization since the adaptive safety weight $\beta$ reflects the safety of the current policy.

In summary, we have an entropy regularized expected return with redefined (regularized) reward
$
    r''_t = r^\targ_t+\beta r^\src_t.
$
This augmented reward encourages the student to yield actions that are more likely to be generated by the guide.
Then, SAC-$\lambda$ can be directly used to solve \eqref{eq:newConOptRewSha} with the additional entropy constraint (\cref{alg:safetransfer}, lines~\ref{lst:safe_update_begin}-\ref{lst:safe_update_end}).

\subsection{Composite Sampling}

To enhance safety and improve the student during training (\cref{alg:safetransfer}, lines \ref{lst:safe_env_step_begin}-\ref{lst:safe_env_step_end}), we leverage a \textit{composite sampling} strategy, which means our behaviour policy ($\pib$) is a mixture of the guide's policy ($\pi^\src$) and the student's policy ($\pi^\targ$).
So, at each environment step,  $a_t \sim \pib(\cdot \mid s_t), s_t \in \mathcal{S}^\targ$ where
\begin{equation}
\label{eq:behaviour_policy}
    \pib(\cdot \mid s_t)
    =
    \begin{cases}
    \pi^\src(\cdot \mid \Xi(s_t)), &\quad \text{if } b = \src, \\
    \pi^\targ(\cdot \mid s_t), &\quad \text{otherwise}.
    \end{cases}
\end{equation}
We investigate two strategies to define $b$.

\textbf{\textbf{Linear-decay (\cref{alg:composite_sampling-ld} in \cref{app:composite-sampling}).}}
{This strategy, denoted as $b = f_{\text{ld}}(\src,\targ)$,} linearly decreases the probability of using $\pi^\src$ with a constant decay rate after each iteration of the algorithm, conversely increasing the probability of using $\pi^\targ$. 
We have two modes with \textit{linear-decay}: 
    \emph{step-wise}, where in each time step we may change $\pib$; and 
    \emph{trajectory-wise}, where $\pib$ only changes at the start of a trajectory.
The mode is decided before executing an episode, and smoothly switches from the complete \emph{step-wise} to the complete \emph{trajectory-wise} over the training process. 

\textbf{\textbf{Control-switch (\cref{alg:composite_sampling-cs} in \cref{app:composite-sampling}).}}
To balance between the safe exploration and the sample efficiency (the samples from the target policy is relatively more valuable), the student policy keeps sampling, i.e.,  $\pib = \pi^\targ$ at the start of a trajectory; after we meet the first $c_{t-1}>0$, we have $\pib = \pi^\src$ until the end of the trajectory. Therefore, the guide policy serves as a \textit{rescue policy} to improve safety during sampling.
We denote this strategy as $b = f_{\text{cs}}(\src,\targ)$.

With the \textit{composite sampling} strategy, the function approximation may diverge, because $\pi^\targ$ and $\pib$ are too different, especially when we collect most data following $\pi^\src$.
This phenomenon is related to the \textit{deadly triad}~\cite{sutton2016emphatic}.
To eliminate its negative effect, we endow each sample with an \textit{importance sampling} (IS) ratio:
\begin{equation}
\label{eq:ISratio}
    \mathcal{I}(s, a) =
    \min\left(\max\left(\frac{\pi^\targ(a\mid s)}{\pib(a\mid s)}, \mathcal{I}_l\right),\mathcal{I}_u\right).
\end{equation}
The clipping hyper-parameters $\mathcal{I}_u$ and $\mathcal{I}_l$ are introduced to reduce the variance of the off-policy TD target.
Notice that if $\pib$ is using the student $\pi^\targ$ then $\mathcal{I}(s,a) = 1$.
Here, in addition to use the IS ratio $\mathcal{I}$ for learning values (the \textit{critics}),
we also use it in the policy update, as shown in line~\ref{lst:updateparameters} of \cref{alg:safetransfer}.

\section{Empirical Analysis}
We evaluate how well our method transfers from the reward-free setting using the SafetyGym engine~\cite{ray2019benchmarking}, where a random-initialized robot navigates in a 2D map to reach target positions while trying to avoid dangerous areas and obstacles (\cref{fig:envs}). These tasks are particularly complex due to the observation space; instead of observing its location, the agent observes the relative location of other objects with a lidar sensor.
We considered three environments with different complexity levels.
A \textbf{static} environment with a point robot and a hazard. The locations of the hazard and goal are fixed in all episodes.   
A \textbf{semi-dynamic} environment with a car robot, four hazards, and four vases. The locations of the hazards and vases are the same in all episodes.
The location of the goal is randomly-initialized in each episode.
A \textbf{dynamic} environment with a point robot, eight hazards, and a vase. The locations of the goal, vase, and hazards are randomly-initialized in each episode.

\begin{figure}[tbp]\hfill
    \includegraphics[width=0.28\columnwidth]{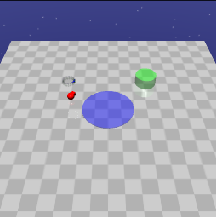}\hfill
    \includegraphics[width=0.28\columnwidth]{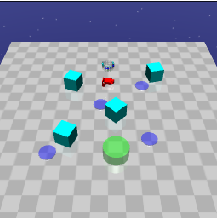}\hfill\includegraphics[width=0.28\columnwidth]{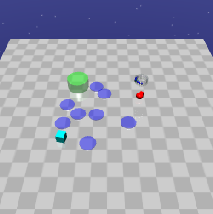}\hfill~
\caption{Navigation tasks with different complexity levels: \textbf{static} where all objects are fixed (left), \textbf{semi-dynamic} where the goal is randomly initialized before each episode (center), and \textbf{dynamic} where
all objects are randomly initialized locations before each episode (right).
}
\label{fig:envs}
\end{figure}

The \textit{guide} agent is trained without the goals, and its auxiliary reward is the magnitude of displacement at each time step.
We provide a detailed description of the safety-mapping function in \cref{app:hyper}.
Since our focus is on the target task and the guide is trained in a controlled environment, we do not consider the guide's training in the evaluation.
In the target tasks, we use the original reward signal from Safety Gym, i.e., the distance towards the goal plus a constant for finishing the task~\cite{ray2019benchmarking}.
In all environments: $c=1$, if an unsafe interaction happens, and $c=0$, otherwise.
We repeat each experiment 10 times with different random seeds and the plots show the mean and standard deviation of all runs.

To evaluate the performance during training, we use the following metrics:
    safety of the behaviour policy (\textrm{Cost-Return~$\pib$}),
    performance of the behaviour policy (\textrm{Return~$\pib$}),
    safety of the target policy (\textrm{Cost-Return~$\pi^\targ$}),
    and performance of the target policy (\textrm{Return~$\pi^\targ$}).
To check the convergence of the target policy, we have a test process with 100 episodes after each epoch (in parallel to the training) to evaluate \textrm{Return~$\pi^\targ$} and \textrm{Cost-Return~$\pi^\targ$}.
\cref{app:target-policy} reports the evaluation of $\pi^\targ$ and \cref{app:hyper} the hyperparameters used.
The supplemental material provides the code of the experiments.

\subsection{Ablation Study}
We investigate each component of the proposed \textsc{\ours{}} algorithm individually to answer the following questions:
    \textit{i}) Does the \textit{auxiliary reward} enlarge the exploration range?
    \textit{ii}) Does a better \textit{guide} agent result in a better student in the target task?
    \textit{iii}) How does the \textit{adaptive strength} of the KL regularization affect the performance?
    \textit{iv}) How does the \textit{composite sampling} benefit the safe transfer learning?

\textbf{\textit{i}) Auxiliary reward leads to more diverse trajectories.}
We performed an ablation of our approach where no auxiliary reward is added while training the \textit{guide} agent, called \textsc{MaxEnt}.
We refer to the agent with the auxiliary reward as \textsc{\ours{}}.
This teases apart the role the designed auxiliary task plays in the exploration.
In Figure~\ref{fig:exploration_analysis}, we can see that \textsc{\ours{}} can explore larger areas in \textit{Static} and \textit{Semi-Dynamic}, which have the same layout in each episode.
We notice that \textsc{MaxEnt} is safe, but the explored space is limited.
That is also the case in \textit{Dynamic}, as shown in the attached videos.

\textbf{\textit{ii}) An effective guide can speed up the student's training.}
We compare how these guides (\textsc{MaxEnt} and \textsc{\ours{}}) affect the learning in the target task. In Figure~\ref{fig:AblationStudy} (\cref{app:ablation-study}),
we notice that both methods can collect samples safely, but the agent using the auxiliary reward needs fewer interactions to find high-performing policies.

\textbf{\textit{iii}) Safety-adaptive regularization improves the student's convergence rate.}
To combine the original reward with the bonus to follow the guide ($\omega$), we have the following choices:
    fix the weights of the bonus and make it to be a hyperparameter to tune (\textsc{FixReg});
    apply a decay rate to linearly decrease the weights during training~(\textsc{DecReg}); and,
    adapt the weights of the bonus based on the safety performance~(\textsc{\ours{}}).
In \cref{at-a} (\cref{app:ablation-study}) we observe that this weight does not affect the safety of the agent, but both \textsc{FixReg} and \textsc{DecReg} cause the student to converge slower in terms of performance (\cref{at-b} in~\cref{app:ablation-study}).

\textbf{\textit{iv}) Composite sampling enhances safety and final performance.}
We modify the composite sampling approach, sampling only from the guide (\textsc{GuiSam}) or the student (\textsc{StuSam}) instead.
From the results in \cref{at-a} (\cref{app:ablation-study}), we can see that \textsc{GuiSam} can ensure safety, but the student does not learn a safe optimal policy (\cref{at-b} in~\cref{app:ablation-study}).
Compared to our method, \textsc{StuSam} performs similarly converging to a safe target policy, but fails to satisfy the constraint at the early stage of training. So, \textit{composite sampling} is necessary to avoid the dangerous actions from a naive policy and to ensure the target task is solved.

\begin{figure}[tbp]
\centering
\setlength\tabcolsep{2pt}\small
\begin{tabular}{ccccc}
\multicolumn{2}{c}{\textbf{Static}} &~& \multicolumn{2}{c}{\textbf{Semi-Dynamic}} \\
\includegraphics[height=0.23\columnwidth]{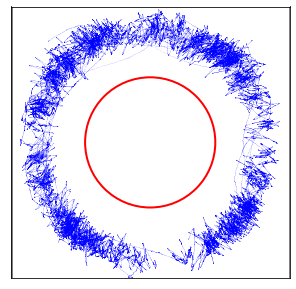}&
\includegraphics[height=0.23\columnwidth]{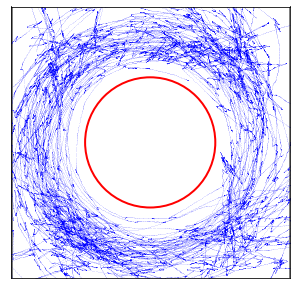}&~&
\includegraphics[height=0.23\columnwidth]{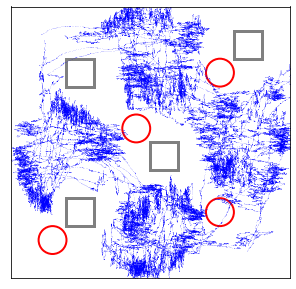}&
\includegraphics[height=0.23\columnwidth]{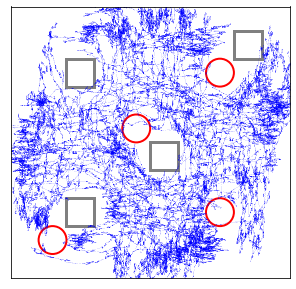}\\
\textsc{MaxEnt}  & \textsc{\ours{}} &~& \textsc{MaxEnt} & \textsc{\ours{}}
\end{tabular}
\caption{Exploration analysis with trajectories collected by the different guide agents in Static and Semi-Dynamic. 
}
\label{fig:exploration_analysis}
\end{figure}

\begin{figure*}[tb]
\subfigure[Static]{\label{DuringTraining-a}
\begin{minipage}[t]{0.33\linewidth}
\includegraphics[width=0.93\textwidth]{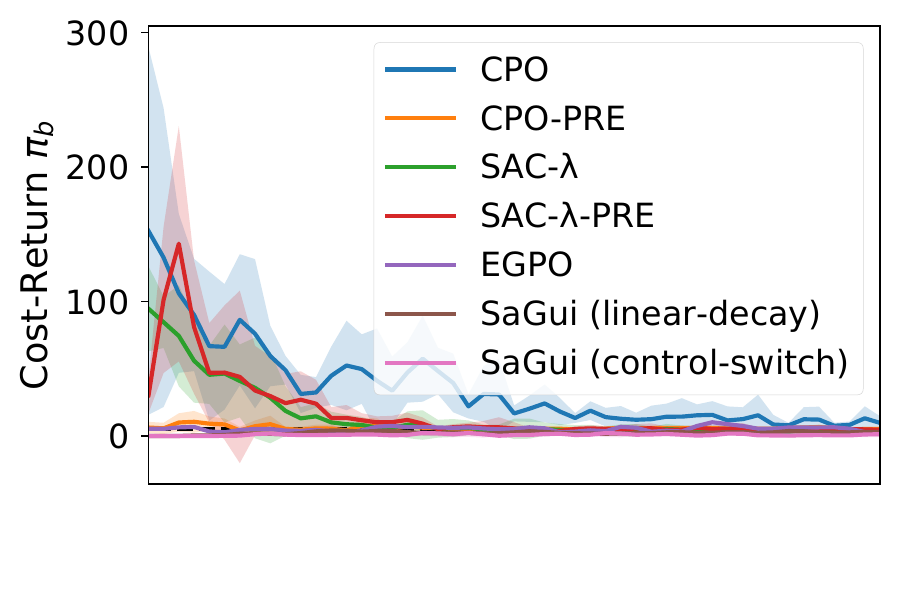}  \vspace{-21.5pt}\\
\includegraphics[width=0.93\textwidth]{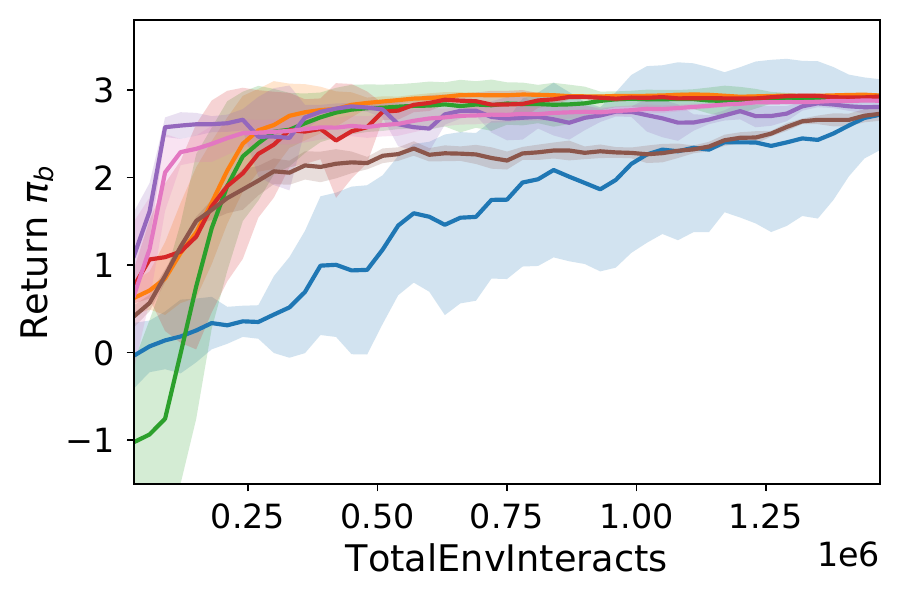}
\end{minipage}
}%
\subfigure[Semi-Dynamic]{\label{DuringTraining-b}
\begin{minipage}[t]{0.33\linewidth}
\includegraphics[width=0.93\textwidth]{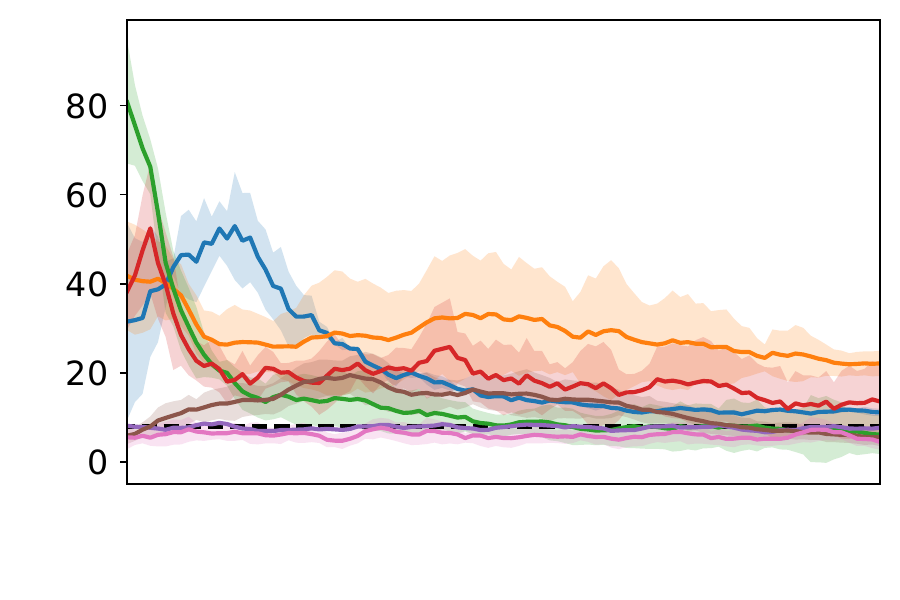} \vspace{-21.5pt}\\ 
\includegraphics[width=0.93\textwidth]{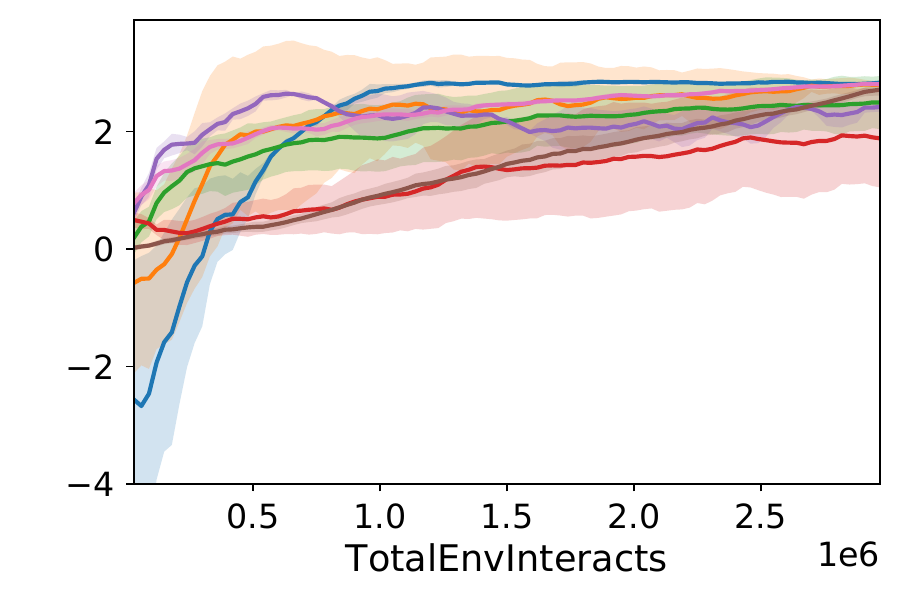}
\end{minipage}
}%
\subfigure[Dynamic]{\label{DuringTraining-c}
\begin{minipage}[t]{0.33\linewidth}
\includegraphics[width=0.93\textwidth]{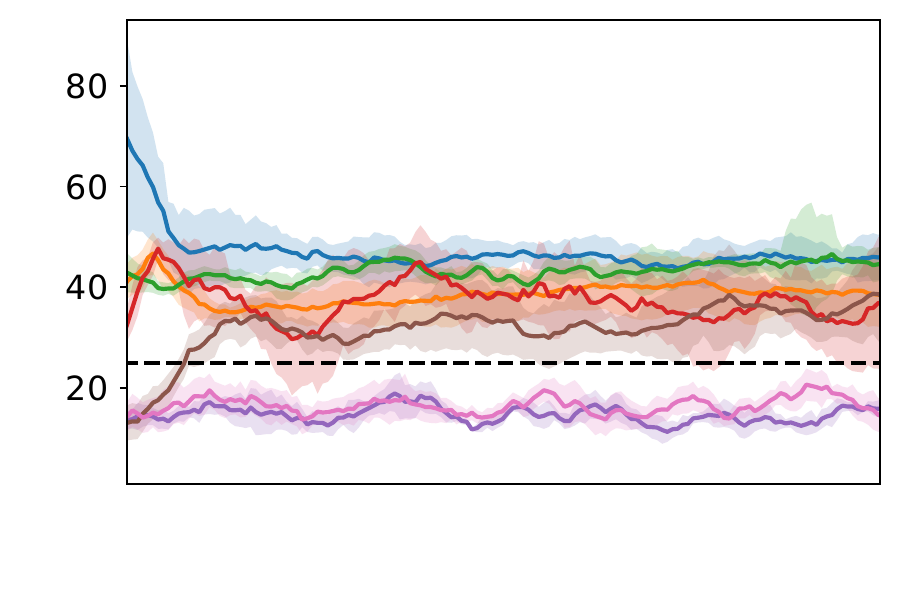}  \vspace{-21.5pt}\\
\includegraphics[width=0.93\textwidth]{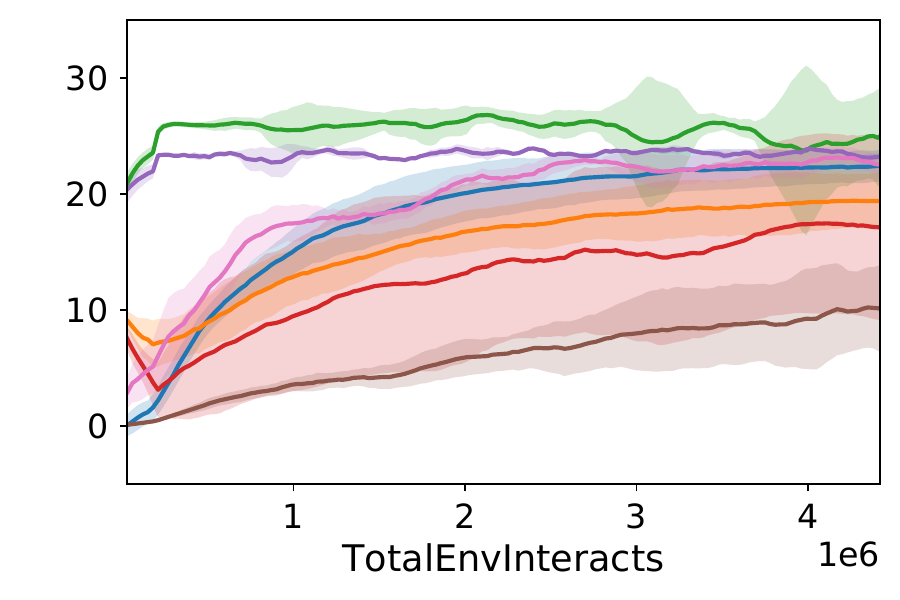}
\end{minipage}
}%
\caption{Evaluation of $\pib$ for CPO, \textsc{CPO-pre}, SAC-$\lambda$, \textsc{SAC-$\lambda$-pre}, EGPO, and \textsc{\ours{}} over 10 seeds. The solid lines are the average of all runs, and the shaded area is the standard deviation. The black dashed lines indicate the safety thresholds.
}
\label{DuringTraining}
\end{figure*}

\subsection{Comparison with Baselines}
Finally, we compare {\footnotesize\color{tab:brown}\faSquare}~\textsc{\ours{}} (control-switch) and {\color{tab:pink}\faSquare}~\textsc{\ours{}} (linear-decay) with five baselines, divided into three groups.
\begin{description}
    \item[Learning from scratch.]
        \textbf{(1)} {\color{tab:green}\faSquare}~\textsc{SAC-$\lambda$} \cite{ha2020learning} shows the performance when starting to learn from scratch, representing an off-policy algorithm.
        Similarly, \textbf{(2)} {\color{tab:blue}\faSquare}~\textsc{CPO} \cite{achiam2017cpo} is an on-policy algorithm that maximizes the reward in a small neighbourhood to enforce the safety constraints.
    \item[Pre-training.]
        \textbf{(3)} {\color{tab:orange}\faSquare}~\textsc{CPO-pre} and \textbf{(4)}  {\color{tab:red}\faSquare}~\textsc{SAC-$\lambda$-pre} demonstrate how CPO and SAC-$\lambda$ perform after being pre-trained in a task that replaces the target reward by the auxiliary reward.
        So, we also encourage exploration in the task for pre-training, which shares the same observation space with the target task. 
    \item[Expert-in-the-loop.]
        \textbf{(5)} As an upper bound, we also consider the Expert Guided Policy Optimization ({\color{tab:purple}\faSquare}~EGPO)~\cite{peng2022safe} algorithm, which uses knowledge from the target task in the form of an expert to train a student policy.
        EGPO proposes a guardian mechanism that replaces the actions of the student by the expert when the student takes actions too different from the expert.
        In summary, EGPO uses an expert policy as a demonstrator as well as a safety guardian (see \cref{app:egpo} for more details).
\end{description}

Notice, for \textsc{CPO-pre}, \textsc{SAC-$\lambda$-pre} and \textsc{EGPO} we adapt the source task to have the same observation space as the target task, which gives them an advantage compared to \textsc{\ours{}}.
Furthermore, \textsc{EGPO} has access to a policy trained on the target task, while \textsc{\ours{}} only has access to the source task without the goal observations.

\textbf{Safety during training.}
In \cref{DuringTraining}, we observe that \textsc{\ours{}} (control-switch) and \textsc{EGPO} are the only methods that exhibit safe behaviour during the full training process.

\textbf{Learning from scratch is unsafe and may converge to sub-optimal and even unsafe policies.}
\textsc{SAC-$\lambda$} and \textsc{CPO} can learn safe policies in relatively simpler environments (\textit{Static} and \textit{Semi-Dynamic}) but they violate the safety constraints at the beginning of training, which is expected. In \textit{Dynamic}, \textsc{SAC-$\lambda$} and \textsc{CPO} fail to attain safe performance. However, with benefits from the \textit{guide}, \textsc{\ours{}} (control-switch), on the basis of \textsc{SAC-$\lambda$}, attains a better balance between safety and performance.

\textbf{Pre-training is insufficient.}
With pre-training, a safe initialization cannot benefit \textsc{CPO-pre} and \textsc{SAC-$\lambda$-pre} in safety, and may have negative effects. We infer that it is difficult to generalize a task when faced with a new reward signal~\cite{Igl2021}. Especially for \textsc{SAC-$\lambda$-pre} with an initialized~$Q^r$, the difficulty to adapt is evident. 

\textbf{Fast convergence rates.}
Benefiting from the targeted expert policy, the behaviour policy of \textsc{EGPO} has a high return throughout the training in the target environment.
But \textsc{\ours{}} (control-switch) quickly finds policies with similar performance despite lack of knowledge of the target task (\cref{DuringTraining}).

\textbf{The distillation mechanism ensures the safety of the target policy.}
\cref{DuringTrainingTarget} (\cref{app:target-policy}) shows that \textsc{\ours{}} (control-switch) can learn a well-performing target policy in a safe way.
Without the policy distillation mechanism like \textsc{\ours{}}, \textsc{EGPO} (learning only from the expert demonstrations) fails to find a safe target policy.
This indicates that the target policy computed with \textsc{\ours{}} may eventually take full control of the target task, while the policy computed by \textsc{EGPO} may still require interventions from the expert.

\textbf{Control-switch can be more effective than linear-decay.}
\textsc{\ours{}} (linear-decay), which lacks samples from $\pi^\targ$ at the early stage of training, does not achieve similar performance as \textsc{\ours{}} (control-switch). \cref{DuringTraining-b,DuringTraining-c} show that \textit{linear-decay} fails to compose the behaviour policy $\pib$ safely.

\textbf{Summary.}
Overall, \textsc{\ours{}} does not violate the safety constraints on the target environment, quickly finds high-performing policies, and can train a student able to act independently from the guide.

\section{Conclusion}
This work handles multiple challenges of reinforcement learning with safety constraints.
It shows how we can use a safe exploration policy (the guide) during data collection and gradually switch to a policy that is dedicated to the target task (the student).
It tackles the off-policy issue that arises from collecting data with a policy different from the target policy.
It shows how the student can make the best use of the guide's policy using an incentive to imitate the guide, which makes the student learn faster how to behave safely.
It demonstrates that simply initializing an agent with a safe policy may not be as effective as learning a new policy dedicated to the target task through policy distillation.
Finally, it proposes a method that can collect diverse trajectories, which reduces the sample complexity of the student on the target task.
In summary, the framework proposed is a safe and sample-efficient way of training the agent on a target task.

\textbf{Limitations.}
Our framework assumes that the source task provides information on the cost function, allowing the guide policy to accumulate the same cost in the target task as in the source task (\cref{method-overview}). This assumption enables safe learning in the target task. However, if the cost function or trajectory distribution changes, the source task may not provide useful safety information for the target task. In such cases, alternative methods should be considered to ensure safe exploration. We focus on downstream tasks where pre-trained agents are utilized for safe exploration knowledge, disregarding the interactions used to train the SaGui policy. Sample efficiency in the target task is emphasized, not including samples used for source task learning. Nevertheless, the pre-trained policy can be reused for multiple target tasks, enabling us to amortize the guide's training across them, making the number of samples required to train the guide negligible as the number of downstream tasks increases. While efficient learning of a SaGui policy is a significant challenge, we view it as a separate research direction~\cite{hazan2019provably}.

\textbf{Future work.}
While we consider a relatively simple strategy to achieve rich exploration, our framework allows the translation of any progress in reward-free RL into training the \textit{guide} agent.
For instance, we could adopt works with the entropy of the state density~\cite{hazan2019provably,lee2019efficient,seo2021state,islam2019marginalized,Zhang2020,svidchenko2021maximum,vezzani2019learning,qin2021density,Yang2023cem}, or with the adaptive reward functions to explore various skills~\cite{eysenbachdiversity}.
Another option to improve exploration is to find a set of diverse policies to the same problem~\cite{Ghasemi2021,Kumar2020,zahavy2021discovering}. 
Our framework could easily combine multiple guides.
As to composite sampling strategies, recovery and shielding mechanisms \cite{Alshiekh2018,thananjeyan2021recovery} could be further explored to combine with a safe guide, in particular using the control-switch mechanism that we evaluated. 
Nevertheless, we highlight that while a student using a recovery policy must explore alone, the safe guide can enhance the  student's exploration, accelerating the learning of the target task.

%% file: Figures/sagui.pdf_tex
\begingroup%
  \makeatletter%
  \providecommand\color[2][]{%
    \errmessage{(Inkscape) Color is used for the text in Inkscape, but the package 'color.sty' is not loaded}%
    \renewcommand\color[2][]{}%
  }%
  \providecommand\transparent[1]{%
    \errmessage{(Inkscape) Transparency is used (non-zero) for the text in Inkscape, but the package 'transparent.sty' is not loaded}%
    \renewcommand\transparent[1]{}%
  }%
  \providecommand\rotatebox[2]{#2}%
  \newcommand*\fsize{\dimexpr\f@size pt\relax}%
  \newcommand*\lineheight[1]{\fontsize{\fsize}{#1\fsize}\selectfont}%
  \ifx\svgwidth\undefined%
    \setlength{\unitlength}{680.94399279bp}%
    \ifx\svgscale\undefined%
      \relax%
    \else%
      \setlength{\unitlength}{\unitlength * \real{\svgscale}}%
    \fi%
  \else%
    \setlength{\unitlength}{\svgwidth}%
  \fi%
  \global\let\svgwidth\undefined%
  \global\let\svgscale\undefined%
  \makeatother%
  \begin{picture}(1,0.26899217)%
    \lineheight{1}%
    \setlength\tabcolsep{0pt}%
    \put(0,0){\includegraphics[width=\unitlength,page=1]{sagui.pdf}}%
    \put(0.15338254,0.21359418){\color[rgb]{0.2,0.2,0.2}\makebox(0,0)[t]{\lineheight{1.25}\smash{\begin{tabular}[t]{c}$s,c$\end{tabular}}}}%
    \put(0.14761948,0.0537289){\color[rgb]{0.2,0.2,0.2}\makebox(0,0)[t]{\lineheight{1.25}\smash{\begin{tabular}[t]{c}$a$\end{tabular}}}}%
    \put(0.64740468,0.21359418){\color[rgb]{0.2,0.2,0.2}\makebox(0,0)[t]{\lineheight{1.25}\smash{\begin{tabular}[t]{c}$s,r,c$\end{tabular}}}}%
    \put(0.74194103,0.09391253){\color[rgb]{0.2,0.2,0.2}\makebox(0,0)[t]{\lineheight{1.25}\smash{\begin{tabular}[t]{c}\large$\pib$\end{tabular}}}}%
    \put(0.64164162,0.0537289){\color[rgb]{0.2,0.2,0.2}\makebox(0,0)[t]{\lineheight{1.25}\smash{\begin{tabular}[t]{c}$a$\end{tabular}}}}%
    \put(0.2487878,0.11878445){\color[rgb]{0.12156863,0.47058824,0.70588235}\makebox(0,0)[t]{\lineheight{1.25}\smash{\begin{tabular}[t]{c}\Large$\pi^\src$\end{tabular}}}}%
    \put(0.42067308,0.18085937){\color[rgb]{0.12156863,0.47058824,0.70588235}\makebox(0,0)[t]{\lineheight{1.25}\smash{\begin{tabular}[t]{c}\Large$\pi^\src$\end{tabular}}}}%
    \put(0.16520037,0.00346403){\color[rgb]{0.12156863,0.47058824,0.70588235}\makebox(0,0)[t]{\lineheight{1.25}\smash{\begin{tabular}[t]{c}\small source (controlled environment)\end{tabular}}}}%
    \put(0.65933683,0.00346403){\color[rgb]{0.2,0.62745098,0.17254902}\makebox(0,0)[t]{\lineheight{1.25}\smash{\begin{tabular}[t]{c}\small target (real world)\end{tabular}}}}%
    \put(0.70002486,0.15184411){\color[rgb]{0.12156863,0.47058824,0.70588235}\makebox(0,0)[t]{\lineheight{1.25}\smash{\begin{tabular}[t]{c}$\pi^\src$\end{tabular}}}}%
    \put(0.78765167,0.15173003){\color[rgb]{0.2,0.62745098,0.17254902}\makebox(0,0)[t]{\lineheight{1.25}\smash{\begin{tabular}[t]{c}$\pi^\targ$\end{tabular}}}}%
    \put(0.05712591,0.12944948){\color[rgb]{0.12156863,0.47058824,0.70588235}\makebox(0,0)[t]{\lineheight{1.25}\smash{\begin{tabular}[t]{c}\large$\mathcal{M}^{\src}$\end{tabular}}}}%
    \put(0.55332885,0.12959566){\color[rgb]{0.2,0.62745098,0.17254902}\makebox(0,0)[t]{\lineheight{1.25}\smash{\begin{tabular}[t]{c}\large$\mathcal{M}^{\targ}$\end{tabular}}}}%
    \put(0.89038662,0.24697332){\color[rgb]{0.2,0.2,0.2}\makebox(0,0)[lt]{\lineheight{1.25}\smash{\begin{tabular}[t]{l}\small transfer\end{tabular}}}}%
    \put(0.89038662,0.20442609){\color[rgb]{0.2,0.2,0.2}\makebox(0,0)[lt]{\lineheight{1.25}\smash{\begin{tabular}[t]{l}\small distillation\end{tabular}}}}%
    \put(0.89038662,0.16187884){\color[rgb]{0.2,0.2,0.2}\makebox(0,0)[lt]{\lineheight{1.25}\smash{\begin{tabular}[t]{l}\small composition\end{tabular}}}}%
  \end{picture}%
\endgroup%

%% file: Figures/safety_transfer_metrics.pdf_tex
\begingroup%
  \makeatletter%
  \providecommand\color[2][]{%
    \errmessage{(Inkscape) Color is used for the text in Inkscape, but the package 'color.sty' is not loaded}%
    \renewcommand\color[2][]{}%
  }%
  \providecommand\transparent[1]{%
    \errmessage{(Inkscape) Transparency is used (non-zero) for the text in Inkscape, but the package 'transparent.sty' is not loaded}%
    \renewcommand\transparent[1]{}%
  }%
  \providecommand\rotatebox[2]{#2}%
  \newcommand*\fsize{\dimexpr\f@size pt\relax}%
  \newcommand*\lineheight[1]{\fontsize{\fsize}{#1\fsize}\selectfont}%
  \ifx\svgwidth\undefined%
    \setlength{\unitlength}{472.09213257bp}%
    \ifx\svgscale\undefined%
      \relax%
    \else%
      \setlength{\unitlength}{\unitlength * \real{\svgscale}}%
    \fi%
  \else%
    \setlength{\unitlength}{\svgwidth}%
  \fi%
  \global\let\svgwidth\undefined%
  \global\let\svgscale\undefined%
  \makeatother%
  \begin{picture}(1,0.50289586)%
    \lineheight{1}%
    \setlength\tabcolsep{0pt}%
    \put(0,0){\includegraphics[width=\unitlength,page=1]{safety_transfer_metrics.pdf}}%
    \put(0.19955368,0.47273649){\color[rgb]{0.51764706,0.51764706,0.51764706}\makebox(0,0)[rt]{\lineheight{1.27636385}\smash{\begin{tabular}[t]{r}cost-return\end{tabular}}}}%
    \put(0.98750824,0.01328977){\color[rgb]{0.51764706,0.51764706,0.51764706}\makebox(0,0)[rt]{\lineheight{1.27636385}\smash{\begin{tabular}[t]{r}episode\end{tabular}}}}%
    \put(0.53215543,0.0769539){\color[rgb]{0,0,0}\makebox(0,0)[t]{\lineheight{1.27636385}\smash{\begin{tabular}[t]{c}$\Delta$ time to safety\end{tabular}}}}%
    \put(0.19908771,0.35637593){\color[rgb]{0,0,0}\makebox(0,0)[rt]{\lineheight{0.92045474}\smash{\begin{tabular}[t]{r}safety\\jump-start\end{tabular}}}}%
    \put(0.35536452,0.4042518){\color[rgb]{0.2,0.62745098,0.17254902}\makebox(0,0)[lt]{\lineheight{0.92045474}\smash{\begin{tabular}[t]{l}learning from\\scratch\end{tabular}}}}%
    \put(0.27207695,0.24235368){\color[rgb]{0.12156863,0.47058824,0.70588235}\makebox(0,0)[lt]{\lineheight{1.27636385}\smash{\begin{tabular}[t]{l}transfer\end{tabular}}}}%
    \put(0.99015494,0.21135378){\color[rgb]{0.69411765,0.34901961,0.15686275}\makebox(0,0)[rt]{\lineheight{0.92045474}\smash{\begin{tabular}[t]{r}safety\\threshold\end{tabular}}}}%
  \end{picture}%
\endgroup%

%% file: Figures/safety_transfer_metrics_safe_initial_policy.pdf_tex
\begingroup%
  \makeatletter%
  \providecommand\color[2][]{%
    \errmessage{(Inkscape) Color is used for the text in Inkscape, but the package 'color.sty' is not loaded}%
    \renewcommand\color[2][]{}%
  }%
  \providecommand\transparent[1]{%
    \errmessage{(Inkscape) Transparency is used (non-zero) for the text in Inkscape, but the package 'transparent.sty' is not loaded}%
    \renewcommand\transparent[1]{}%
  }%
  \providecommand\rotatebox[2]{#2}%
  \newcommand*\fsize{\dimexpr\f@size pt\relax}%
  \newcommand*\lineheight[1]{\fontsize{\fsize}{#1\fsize}\selectfont}%
  \ifx\svgwidth\undefined%
    \setlength{\unitlength}{472.09213257bp}%
    \ifx\svgscale\undefined%
      \relax%
    \else%
      \setlength{\unitlength}{\unitlength * \real{\svgscale}}%
    \fi%
  \else%
    \setlength{\unitlength}{\svgwidth}%
  \fi%
  \global\let\svgwidth\undefined%
  \global\let\svgscale\undefined%
  \makeatother%
  \begin{picture}(1,0.50289586)%
    \lineheight{1}%
    \setlength\tabcolsep{0pt}%
    \put(0,0){\includegraphics[width=\unitlength,page=1]{safety_transfer_metrics_safe_initial_policy.pdf}}%
    \put(0.20118466,0.47273649){\color[rgb]{0.52156863,0.52156863,0.52156863}\makebox(0,0)[rt]{\lineheight{1.27636385}\smash{\begin{tabular}[t]{r}cost-return\end{tabular}}}}%
    \put(0.98750824,0.01328977){\color[rgb]{0.52156863,0.52156863,0.52156863}\makebox(0,0)[rt]{\lineheight{1.27636385}\smash{\begin{tabular}[t]{r}episode\end{tabular}}}}%
    \put(0.99015494,0.21135092){\color[rgb]{0.69411765,0.34901961,0.15686275}\makebox(0,0)[rt]{\lineheight{0.92045474}\smash{\begin{tabular}[t]{r}safety\\threshold\end{tabular}}}}%
    \put(0.4592743,0.18447301){\color[rgb]{0,0,0}\makebox(0,0)[t]{\lineheight{1.27636385}\smash{\begin{tabular}[t]{c}$\Delta$ time to safety\end{tabular}}}}%
    \put(0.41005339,0.08142299){\color[rgb]{0.12156863,0.47058824,0.70588235}\makebox(0,0)[lt]{\lineheight{1.27636385}\smash{\begin{tabular}[t]{l}transfer\end{tabular}}}}%
    \put(0.20071868,0.30891372){\color[rgb]{0,0,0}\makebox(0,0)[rt]{\lineheight{0.92045474}\smash{\begin{tabular}[t]{r}safety\\jump-start\end{tabular}}}}%
    \put(0.32312074,0.42787064){\color[rgb]{0.2,0.62745098,0.17254902}\makebox(0,0)[lt]{\lineheight{0.92045474}\smash{\begin{tabular}[t]{l}learning from\\scratch\end{tabular}}}}%
  \end{picture}%
\endgroup%

%% file: Figures/safety_transfer_metrics_return.pdf_tex
\begingroup%
  \makeatletter%
  \providecommand\color[2][]{%
    \errmessage{(Inkscape) Color is used for the text in Inkscape, but the package 'color.sty' is not loaded}%
    \renewcommand\color[2][]{}%
  }%
  \providecommand\transparent[1]{%
    \errmessage{(Inkscape) Transparency is used (non-zero) for the text in Inkscape, but the package 'transparent.sty' is not loaded}%
    \renewcommand\transparent[1]{}%
  }%
  \providecommand\rotatebox[2]{#2}%
  \newcommand*\fsize{\dimexpr\f@size pt\relax}%
  \newcommand*\lineheight[1]{\fontsize{\fsize}{#1\fsize}\selectfont}%
  \ifx\svgwidth\undefined%
    \setlength{\unitlength}{472.09213257bp}%
    \ifx\svgscale\undefined%
      \relax%
    \else%
      \setlength{\unitlength}{\unitlength * \real{\svgscale}}%
    \fi%
  \else%
    \setlength{\unitlength}{\svgwidth}%
  \fi%
  \global\let\svgwidth\undefined%
  \global\let\svgscale\undefined%
  \makeatother%
  \begin{picture}(1,0.50289586)%
    \lineheight{1}%
    \setlength\tabcolsep{0pt}%
    \put(0,0){\includegraphics[width=\unitlength,page=1]{safety_transfer_metrics_return.pdf}}%
    \put(0.30939249,0.30625976){\color[rgb]{0.12156863,0.47058824,0.70588235}\makebox(0,0)[lt]{\lineheight{0.92045474}\smash{\begin{tabular}[t]{l}guided\\exploration\end{tabular}}}}%
    \put(0.98750821,0.01328979){\color[rgb]{0.52156863,0.52156863,0.52156863}\makebox(0,0)[rt]{\lineheight{1.27636385}\smash{\begin{tabular}[t]{r}episode\end{tabular}}}}%
    \put(0.51921648,0.39227726){\color[rgb]{0,0,0}\makebox(0,0)[t]{\lineheight{1.27636385}\smash{\begin{tabular}[t]{c}$\Delta$ time to optimum\end{tabular}}}}%
    \put(0.20025271,0.47273649){\color[rgb]{0.52156863,0.52156863,0.52156863}\makebox(0,0)[rt]{\lineheight{1.27636385}\smash{\begin{tabular}[t]{r}return\end{tabular}}}}%
    \put(0.20071867,0.13953097){\color[rgb]{0,0,0}\makebox(0,0)[rt]{\lineheight{0.92045474}\smash{\begin{tabular}[t]{r}return\\jump-start\end{tabular}}}}%
    \put(0.35854187,0.12878135){\color[rgb]{0.2,0.62745098,0.17254902}\makebox(0,0)[lt]{\lineheight{0.92045474}\smash{\begin{tabular}[t]{l}conservative\\exploration\end{tabular}}}}%
    \put(0.99295093,0.42840203){\color[rgb]{0.69411765,0.34901961,0.15686275}\makebox(0,0)[rt]{\lineheight{0.92045474}\smash{\begin{tabular}[t]{r}optimal\\performance\end{tabular}}}}%
  \end{picture}%
\endgroup%

%% file: algorithms/guided_safe_exploration.tex
\begin{algorithm}[t]
\caption{Guided Safe Exploration}
\label{alg:safetransfer}
\textbf{Input}: $\mathcal{M}^\targ$, $\pi^\src$, $\minH$, $d$\\
\textbf{Initialize}: $\mathcal{D} \leftarrow \emptyset$, $\theta^\targ_\chi \text{ for } \chi \in \{\pi,R,C,\alpha,\beta\}$ \\
\textbf{Output}: Optimized parameters $\theta^\targ_\pi$ for $\pi^\targ$

\begin{algorithmic}[1] %
\FOR{each iteration}
        \FOR{each environment step} \label{lst:safe_env_step_begin}
        \IF{\emph{linear-decay}} 
            \STATE $b \leftarrow f_{\text{ld}}(\src,\targ)$ \COMMENT{linearly eliminate the effect of $\pi^\src$}
        \ELSIF{\emph{control-switch}} 
            \STATE $b \leftarrow f_{\text{cs}}(\src,\targ)$ \COMMENT{$\pi^\src$ takes control if unsafe}
            \ENDIF
            \STATE $a_t \sim \pib(\cdot \mid s_t)$
            \COMMENT{Composite sampling \eqref{eq:behaviour_policy}}
            \STATE $\mathcal{I}_t \leftarrow \mathcal{I}(s_t, a_t)$   \COMMENT{IS ratio \eqref{eq:ISratio}}
            \STATE $r^\targ_t \leftarrow r^\targ(s_t,a_t)$
            \STATE $r^\src_t \leftarrow \log \pi^\src(a_t \mid \Xi(s_t))$
            \STATE $c^\targ_t \leftarrow c^\targ(s_t,a_t)$
            \STATE $s_{t+1} \sim \mathcal{P}^\targ(\cdot \mid s_t,a_t)$
            \STATE $\mathcal{D} \leftarrow  \mathcal{D} \cup \{(s_t,a_t,r^\targ_t,r^\src_t,c^\targ_t,\mathcal{I}_t,s_{t+1}) \}$ \label{lst:safe_env_step_end}
        \ENDFOR 
        \FOR{each gradient step}
        \label{lst:safe_update_begin}
        \STATE Sample experience from $\mathcal{D}$ \label{lst:sample_experiences}
        \FOR{$\chi \in \{\pi,R,C,\alpha,\beta\}$}
        \STATE $\theta^\targ_\chi \leftarrow \theta^\targ_\chi - \eta_\chi \hat{\nabla}_{\theta^\targ_\chi} \mathcal{I} J_{\chi}(\theta^\targ_\chi)$ \COMMENT{Updating $\theta^\targ_\chi$} \label{lst:updateparameters} 
        \label{lst:safe_update_end}
        \ENDFOR
        \ENDFOR 
        \ENDFOR %
\end{algorithmic}
\end{algorithm}

%% file: Figures/schematics.pdf_tex
\begingroup%
  \makeatletter%
  \providecommand\color[2][]{%
    \errmessage{(Inkscape) Color is used for the text in Inkscape, but the package 'color.sty' is not loaded}%
    \renewcommand\color[2][]{}%
  }%
  \providecommand\transparent[1]{%
    \errmessage{(Inkscape) Transparency is used (non-zero) for the text in Inkscape, but the package 'transparent.sty' is not loaded}%
    \renewcommand\transparent[1]{}%
  }%
  \providecommand\rotatebox[2]{#2}%
  \newcommand*\fsize{\dimexpr\f@size pt\relax}%
  \newcommand*\lineheight[1]{\fontsize{\fsize}{#1\fsize}\selectfont}%
  \ifx\svgwidth\undefined%
    \setlength{\unitlength}{1193.81628322bp}%
    \ifx\svgscale\undefined%
      \relax%
    \else%
      \setlength{\unitlength}{\unitlength * \real{\svgscale}}%
    \fi%
  \else%
    \setlength{\unitlength}{\svgwidth}%
  \fi%
  \global\let\svgwidth\undefined%
  \global\let\svgscale\undefined%
  \makeatother%
  \begin{picture}(1,0.29298855)%
    \lineheight{1}%
    \setlength\tabcolsep{0pt}%
    \put(0,0){\includegraphics[width=\unitlength,page=1]{schematics.pdf}}%
    \put(0.10438936,0.02806116){\color[rgb]{0,0,0}\makebox(0,0)[rt]{\lineheight{1.25}\smash{\begin{tabular}[t]{r}Student\end{tabular}}}}%
    \put(0.10512557,0.18850471){\color[rgb]{0,0,0}\makebox(0,0)[rt]{\lineheight{1.25}\smash{\begin{tabular}[t]{r}Guide\end{tabular}}}}%
    \put(0.1278376,0.233859){\color[rgb]{0,0,0}\makebox(0,0)[lt]{\lineheight{1.25}\smash{\begin{tabular}[t]{l}$s^\src$\end{tabular}}}}%
    \put(0.12559258,0.08971838){\color[rgb]{0,0,0}\makebox(0,0)[lt]{\lineheight{1.25}\smash{\begin{tabular}[t]{l}$s^\targ$\end{tabular}}}}%
    \put(0.22203471,0.10278005){\color[rgb]{0,0,0}\makebox(0,0)[lt]{\lineheight{1.25}\smash{\begin{tabular}[t]{l}$\Xi(s^\targ)$\end{tabular}}}}%
    \put(0.32115176,0.26677276){\color[rgb]{0,0,0}\makebox(0,0)[t]{\lineheight{1.25}\smash{\begin{tabular}[t]{c}Observation\end{tabular}}}}%
    \put(0.84003055,0.26677276){\color[rgb]{0,0,0}\makebox(0,0)[t]{\lineheight{1.25}\smash{\begin{tabular}[t]{c}Reward\end{tabular}}}}%
    \put(0.54152629,0.05337563){\color[rgb]{0,0,0}\makebox(0,0)[lt]{\lineheight{1.25}\smash{\begin{tabular}[t]{l}$ \qquad r^\targ $\end{tabular}}}}%
    \put(0.84012258,0.03214882){\makebox(0,0)[t]{\lineheight{1.25}\smash{\begin{tabular}[t]{c}$r^\targ + \omega r^{\textrm{KL}}$\end{tabular}}}}%
    \put(0.84012257,0.19001029){\makebox(0,0)[t]{\lineheight{1.25}\smash{\begin{tabular}[t]{c}$D_{\text{KL}}(\pi^\src(\cdot| s^\src)\!\parallel\! \pi^\targ(\cdot | s^\targ))$\end{tabular}}}}%
    \put(0.80638535,0.09023214){\color[rgb]{0,0,0}\makebox(0,0)[lt]{\lineheight{1.25}\smash{\begin{tabular}[t]{l}\textit{\small Distillation Bonus}\end{tabular}}}}%
    \put(0.20815797,0.02806116){\color[rgb]{0,0,0}\makebox(0,0)[t]{\lineheight{1.25}\smash{\begin{tabular}[t]{c}safety-related\end{tabular}}}}%
    \put(0.20815797,0.18850471){\color[rgb]{0,0,0}\makebox(0,0)[t]{\lineheight{1.25}\smash{\begin{tabular}[t]{c}safety-related\end{tabular}}}}%
    \put(0.41312402,0.02806116){\color[rgb]{0,0,0}\makebox(0,0)[t]{\lineheight{1.25}\smash{\begin{tabular}[t]{c}reward-related\end{tabular}}}}%
    \put(0.49484957,0.14386376){\makebox(0,0)[t]{\lineheight{1.25}\smash{\begin{tabular}[t]{c}$\pi^\targ(\cdot \mid s^\targ)$\end{tabular}}}}%
    \put(0.49484957,0.21190495){\makebox(0,0)[t]{\lineheight{1.25}\smash{\begin{tabular}[t]{c}$\pi^\src(\cdot \mid s^\src)$\end{tabular}}}}%
    \put(0.80799648,0.12391869){\color[rgb]{0,0,0}\makebox(0,0)[lt]{\lineheight{1.25}\smash{\begin{tabular}[t]{l}$r^{\textrm{KL}}$\end{tabular}}}}%
  \end{picture}%
\endgroup%

%% file: app.tex
\newpage
\appendix

\section{SAC-Lagrangian}
\label{app:saclag}
In this section, we present how we learn the parameters in SAC-$\lambda$. In SAC-$\lambda$, the constrained optimization problem is solved by Lagrangian methods \cite{bertsekas2014constrained}, where an entropy weight $\alpha$ and a safety weight $\beta$ (Lagrange-multipliers) are introduced to the constrained optimization:
\begin{equation}
\max_\pi \min_{\alpha \geq 0} \min_{\beta \geq 0}  f(\pi) - \alpha e(\pi) - \beta g(\pi),
\label{eq:lagopt}
\end{equation}
where
$f(\pi) = \E_{s_0 \sim \iota(\cdot), a_0 \sim \pi(\cdot\mid s_0)}\left[Q_{\pi}^{r}(s_0,a_0)\right]$,
$e(\pi) = \E_{s_t \sim \rho_{\pi}} \left[ \log(\pi(\cdot\mid s_t))  + \minH \right]$ , \text{and \hfill}
$g(\pi) =  \E_{s_0 \sim \iota(\cdot), a_0 \sim \pi(\cdot\mid s_0)}\left[Q_{\pi}^{c}(s_0,a_0) - d\right]$.
In \eqref{eq:lagopt}, the max-min optimization problem can be solved by gradient ascent on $\pi$, and descent on $\alpha$ and~$\beta$.

Initially, SAC-$\lambda$ was developed for local constraints, which means that the safety cost is constrained at each timestep \cite{ha2020learning}.
However, it can be easily generalized to constrain the expected cost-return\footnote{A similar approach can be found at \url{https://github.com/openai/safety-starter-agents}.}. 

Using a similar formulation \cite{haarnoja2018soft2}, we can get the actor loss:
\begin{equation}
    J_\pi(\theta_\pi) =
    - \E_{\begin{subarray}{c} s_t \sim \mathcal{D} \\ a_t \sim \pi(\cdot \mid s_t)\end{subarray}} 
    \left[ 
                Q_{\pi}^r(s_t,a_t) 
                -\alpha\log \pi(a_t \mid s_t) 
                -\beta Q_{\pi}^c(s_t,a_t)
    \right],
\label{eq:UpdatePi}
\end{equation}
where $\mathcal{D}$ is the replay buffer and $\theta_\pi$ indicates the parameters of the policy~$\pi$.

The safety and reward critics (including a bonus for the policy entropy) are, respectively, trained to minimize
\begin{equation}
    J_C(\theta_C) \!=\!\!
    \E_{(s_t,a_t) \sim \mathcal{D}}
    \left[
        \frac{1}{2}
        \left(
                Q^c_{\theta_C}(s_t,a_t) -
                 (c_t \!+\!\gamma Q^c_{\theta_C}(s_{t\!+\!1},a_{t\!+\!1}))
        \right)^2    
    \right]
\label{eq:lossjc}
\end{equation}
and
\begin{equation}
\begin{aligned}
    J_R(\theta_R) = \E_{(s_t,a_t) \sim \mathcal{D}} \Big[\frac{1}{2} (Q^r_{\theta_R}(s_t,a_t)-(r_t  +\gamma (Q^r_{\theta_R}(s_{t+1},a_{t+1})-\alpha \log (\pi(a_{t+1} \mid s_{t+1})))))^2    \Big],
\end{aligned}
\label{eq:Updateqr}
\end{equation}
where $a_{t+1} \sim \pi(\cdot \mid s_{t+1})$, $Q^c$ and $Q^r$ are parameterized by $\theta_C$ and $\theta_R$, respectively.

\input{algorithms/training_safe_explorer}

Finally, 
let $\theta_\alpha$ and $\theta_\beta$ be the parameters learned for the exploration and safety weight such that $\alpha = {\rm softplus}(\theta_\alpha)$ and $\beta = {\rm softplus}(\theta_\beta)$, where
\[
\softplus(x) = \log(\exp(x) + 1).
\]
We can learn $\alpha$ and $\beta$ by minimizing the loss functions:
\begin{align}
J_\alpha(\theta_\alpha) = \E_{\begin{subarray}{c} s_t \sim \mathcal{D}\\ a_t \sim \pi(\cdot \mid s_t)\end{subarray}} \left[ -\alpha(\log(\pi(a_t \mid s_t))+ \minH ) \right],
\end{align}
\text{and}
\begin{align}
J_\beta(\theta_\beta) = \E_{\begin{subarray}{c} s_t \sim \mathcal{D}\\ a_t \sim \pi(\cdot \mid s_t)\end{subarray}} \left[ \beta(d-Q_{\pi}^c(s_t,a_t))\right].
\end{align}
So the corresponding weight will be adjusted if the constraints are violated, that is, if we estimate that the current policy is unsafe or if it does not have enough entropy.

In this paper, we train the \textit{guide} agent by solving the constraint optimization problem \eqref{eq:SafeMERL} based on the auxiliary reward $r^\delta$, defined by \eqref{eq:DefAuxRew}.
Then, we can use SAC-$\lambda$ directly employed to solve \eqref{eq:SafeMERL}, as \cref{alg:safe-explorer} shows.

\section{Relation between source and target tasks}

In this section, we describe the source task given a target task and the mapping from the target task to the source task.

\subsection{State Abstraction}
\label{app:state_abstraction}

To build the source task based on a target task and a mapping $\Xi$ from the target state space to the source state space, we assume $\Xi$ is a state abstraction function \cite{Li2006}.

Let $\mathcal{M}^\targ = \langle \mathcal{S}^\targ, \mathcal{A}^\targ, \mathcal{P}^\targ, r^\targ, c^\targ, d^\targ, \iota^\targ, \gamma \rangle$ be the target task, 
$\mathcal{M}^\src = \langle \mathcal{S}^\src, \mathcal{A}^\src, \mathcal{P}^\src, \emptyset, c^\src, d^\src, \iota^\src, \gamma \rangle$ be the source task,  and
$\Xi: \mathcal{S}^\targ \rightarrow \mathcal{S}^\src$ the state abstraction function.
We define $\Xi^{-1}$ as the inverse of the abstraction function such that  $\Xi^{-1}(s^\targ) = \{s^\src \in \mathcal{S}^\src | \Xi(s^\src) = s^\targ \}$.
We assume a weighting function $w\colon \mathcal{S} \mapsto [0,1]$,
where 
\begin{equation}\label{eq:w}
\sum_{s^\targ \in \Xi^{-1}(s^\src)} w(s^\targ) = 1, \forall s^\src \in \mathcal{S}^\src.    
\end{equation}
Now we can define the transition and cost function of the target task:
\begin{align}
\mathcal{P}^{\src}(s^{\src'} \mid s^\src, a)
    ~&= \sum_{s^\targ \in \Xi^{-1}(s^\src)} \quad \sum_{s^{\targ'} \in \Xi^{-1}(s^{\src'})} w(s^\targ) \mathcal{P}^{\targ}(s^{\targ'} \mid s^\targ, a)\\
c^\src(s^\src, a)
    ~&= \sum_{s^\targ \in \Xi^{-1}(s^\src) } w(s^\targ) c^\targ(s^\targ, a)\\
\iota^\src(s^\src)
    ~&= \sum_{s^\targ \in \Xi^{-1}(s^\src)} w(s^\targ) \iota^\targ(s^\targ).
\end{align}

\subsection{Proof of Lemma 1}
\label{sec:proof_lemma}

In this section, we show that if $\Xi$ is a $Q_{\pi}^{c}$-irrelevance state abstraction, then the expected cost of any source policy is the same in the source task and in the target task.
For the convenience of the reader, we restate our assumption and lemma.

\paragraph{Assumption 3.}
\textit{
$\Xi$ is a $Q_{\pi}^{c}$-irrelevance abstraction~\cite{Li2006}, therefore
\[
\Xi(s) = \Xi(s')
        \Rightarrow Q_{\pi^\targ}^{c}(s, a) = Q_{\pi^\targ}^{c}(s', a), 
    \forall s, s' \in \mathcal{S}^\targ, a \in \mathcal{A}, \pi^\targ.
\]
}

\paragraph{Lemma 1.}\textit{
Given \cref{a:shared_action} and \cref{a:abstraction}, we have
\[
    Q_{\pi^\src}^{c,\src}(\Xi(s),a) 
    =
    Q_{\pi^{\src \rightarrow \targ}}^{c,\targ}(s,a)
    \quad \forall s \in \mathcal{S}^\targ, a \in \mathcal{A}, \pi^\src.
\]
That is, the expected cost of a source policy is the same in the source task and in the target task.
}
Our proof follows an induction strategy inspired by previous work \cite[Claim 1]{DBLP:conf/icml/AbelHL16}.
\begin{proof}
Let us consider a non-Markovian constrained decision process $\mathcal{M}_T = \langle \mathcal{S}_T, \mathcal{A}, \mathcal{P}_T, \emptyset, c^T, d^\src, \iota_T, \gamma \rangle$ which is parameterized by an integer $T$.
In this process, the agent takes $T$ steps on the source task and then switches to the target task.
Thus,
\begin{align}
\mathcal{S}_T ~&=
\begin{cases}
    \mathcal{S}^\targ & \text{ if } T = 0 \\
    \mathcal{S}^\src & \text{ otherwise.}
\end{cases} \\ 
c_T(s, a) ~&= 
\begin{cases}
    c^\targ(s, a) & \text{ if } T = 0 \\
    c^\src(s, a) & \text{ otherwise.}
\end{cases} \\     
\mathcal{P}_T(s' \mid s, a)~&= 
\begin{cases}
    \mathcal{P}^{\targ}(s' \mid s, a) & \text{ if } T = 0 \\
    \sum_{s^\targ \in \Xi^{-1}(s)} w(s^\targ) \mathcal{P}^{\targ}(s' \mid s^\targ, a) & \text{ if } T = 1 \\
    \mathcal{P}^{\src}(s' \mid s, a) & \text{ otherwise. }
\end{cases} \\
\iota_T(s) ~&= 
\begin{cases}
    \iota^\targ(s) & \text{ if } T = 0 \\
    \iota^\src(s) & \text{ otherwise.}
\end{cases} 
\end{align}
The $Q_{\pi}^{c,\targ}(s,a)$-value for taking action $a \in \mathcal{A}$ in state $s \in \mathcal{S}_T$ and follow the policy $\pi$ is:
\begin{align}
Q_{T,\pi}^{c}(s,a) = 
\begin{cases}
    Q_{\pi}^{c,\targ}(s, a) & \text{ if } T = 0 \\
    \sum_{s^\targ \in \Xi^{-1}(s) } w(s^\targ) Q_{\pi}^{c,\targ}(s^\targ,a) & \text{ if } T = 1 \\
    c^\src(s, a) + \gamma \sum_{s' \in \mathcal{S}^\src} \mathcal{P}^\src(s' \mid s, a)  \sum_{a' \in \mathcal{A}}  \pi(a' \mid s') Q_{T-1,\pi}^{c}(s',a')  & \text{ otherwise.}
\end{cases}
\end{align}
We proceed by induction on $T$ to show that 
\[
\forall T, s^\targ, a, \pi: Q_{\pi}^{c,\targ}(s_T, a) = Q_{\pi}^{c,\targ}(s^\targ,a),
\]
where $s_T = s^\targ$ if $T=0$ and
$s_T = \Xi(s^\targ)$ otherwise.

\paragraph{Base case: $T=0$.}
As $Q_{0}^{c} =  Q^{c,\targ}$ this case follows trivially.

\paragraph{Base case: $T=1$.}
From the definition of $Q_{1,\pi}^{c}$, we have:
\begin{align}
    Q_{1,\pi}^{c}(s_T,a)
        = \sum_{s^{\targ'} \in \Xi^{-1}(s_T) } w(s^{\targ'}) Q_{\pi}^{c,\targ}(s^{\targ'},a) \\
        &= \sum_{s^{\targ'} \in \Xi^{-1}(s_T) } w(s^{\targ'}) Q_{\pi}^{c,\targ}(s^\targ,a) \label{step:s_star}\\
        &= Q_{\pi}^{c,\targ}(s^\targ,a) \sum_{s^{\targ'} \in \Xi^{-1}(s) } w(s^{\targ'}) \label{step:common_term}\\
        &= Q_{\pi}^{c,\targ}(s^\targ,a). \label{step:w}
\end{align}
In \cref{step:s_star}, we replace every $s^{\targ'}$ by the state $s^\targ$ applying Assumption 3.
As $s^\targ$ is independent of $s^{\targ'}$, in \cref{step:common_term}, we can move the Q-values out of the summation.
Finally, in \cref{step:w}, we can use \cref{eq:w} to replace the summation by 1, which concludes this case.

\paragraph{Incuctive case: $T> 1$.}
We assume as our inductive hypothesis that:
\[
\forall s^\targ, a, \pi: Q_{T-1, \pi}^{c}(s_{T}, a) = Q_{\pi}^{c,\targ}(s^\targ,a).
\]

We start applying the definition of $Q_T$ for $T > 1$:
\begin{align}
Q_{T, \pi}^{c}&(s_{T}, a) =
c^\src(s_{T}, a)
 + \gamma \sum_{s' \in \mathcal{S}^\src} \mathcal{P}^\src(s' \mid s_{T}, a) 
 \sum_{a' \in \mathcal{A}}  \pi(a' \mid s') Q_{T-1,\pi}^{c}(s',a')  \\
&=
\sum_{s^\targ \in \Xi^{-1}(s_{T}) } w(s^\targ)  c^\targ(s^\targ, a)
 + \gamma  
\sum_{s' \in \mathcal{S}^\src} \sum_{s^\targ \in \Xi^{-1}(s_{T})}  \sum_{s^{\targ'} \in \Xi^{-1}(s')} w(s^\targ) \mathcal{P}^{\targ}(s^{\targ'} \mid s^\targ, a)
 \sum_{a' \in \mathcal{A}}  \pi(a' \mid s') Q_{T-1,\pi}^{c}(s',a')  \label{step:definitions}\\
&=
\sum_{s^\targ \in \Xi^{-1}(s_{T}) } w(s^\targ)  c^\targ(s^\targ, a)
 + 
\sum_{s^\targ \in \Xi^{-1}(s_{T})} w(s^\targ) \gamma  \sum_{s' \in \mathcal{S}^\src} \sum_{s^{\targ'} \in \Xi^{-1}(s')}  \mathcal{P}^{\targ}(s^{\targ'} \mid s^\targ, a)
 \sum_{a' \in \mathcal{A}}  \pi(a' \mid s') Q_{T-1,\pi}^{c}(s',a')  \label{step:rearange1}\\
&=
\sum_{s^\targ \in \Xi^{-1}(s_{T}) } w(s^\targ) \left[ c^\targ(s^\targ, a)
 +  \gamma  \sum_{s' \in \mathcal{S}^\src} \sum_{s^{\targ'} \in \Xi^{-1}(s')}  \mathcal{P}^{\targ}(s^{\targ'} \mid s^\targ, a)
 \sum_{a' \in \mathcal{A}}  \pi(a' \mid s') Q_{T-1,\pi}^{c}(s',a')  \right] \label{step:rearange2}\\
&=
\sum_{s^\targ \in \Xi^{-1}(s_{T}) } w(s^\targ) \left[ c^\targ(s^\targ, a)
 +  \gamma  \sum_{s' \in \mathcal{S}^\src} \sum_{s^{\targ'} \in \Xi^{-1}(s')}  \mathcal{P}^{\targ}(s^{\targ'} \mid s^\targ, a)
 \sum_{a' \in \mathcal{A}}  \pi(a' \mid s') Q_{\pi}^{c,\targ}(s^{\targ'},a')
 \right]  \label{step:hyp}\\
&=
\sum_{s^\targ \in \Xi^{-1}(s_{T}) } w(s^\targ) \left[ c^\targ(s^\targ, a)
 +  \gamma  \sum_{s^{\targ'} \in \mathcal{S}^\targ}  \mathcal{P}^{\targ}(s^{\targ'} \mid s^\targ, a)
 \sum_{a' \in \mathcal{A}}  \pi(a' \mid s')  Q_{\pi}^{c,\targ}(s^{\targ'},a')
 \right]  \label{step:join_sums}\\
&=
\sum_{s^\targ \in \Xi^{-1}(s_{T}) } w(s^\targ) Q_{\pi}^{c,\targ}(s^\targ,a) \label{step:q_value_definition}\\
& = Q_{\pi}^{c,\targ}(s^\targ,a). \label{step:arbitrary_state}
\end{align}
In this derivation, 
\cref{step:definitions} applies the definitions of $c^\src$ and $\mathcal{P}^\src$,
\cref{step:rearange1,step:rearange2} rearrange our terms,
\cref{step:hyp} applies our inductive hypothesis, \cref{step:join_sums} join the two summations as we are considering all possible states in $\mathcal{S}^\targ$, and \cref{step:q_value_definition} we apply the Q-value definition.
Finally, in \cref{step:arbitrary_state} we can choose any arbitrary state $s^\targ \in \Xi^{-1}(s_{T})$, which concludes our proof.
\end{proof}

\section{Regularized Reward}
\label{app:regularized_reward}

\begin{equation}
\begin{aligned}
    \omega r^{\textrm{KL}} {+} \alpha r^\mathcal{H}
= &\omega \log \frac{\pi^\src(a_t\mid \Xi(s_t))}{\pi^\targ(a_t\mid s_t)} + \omega r^\mathcal{H}\\
= &\omega (\log(\pi^\src(a| \Xi(s))) - \log( \pi^\targ(a| s))) + \alpha r^\mathcal{H}\\
= &\omega \log(\pi^\src(a| \Xi(s))) + \omega (-\log(\pi^\targ(a| s))) + \alpha r^\mathcal{H}\\
= &\omega \log(\pi^\src(a| \Xi(s))) + \omega r^\mathcal{H} + \alpha r^\mathcal{H}\\
= &\omega r^\src + (\omega + \alpha) r^\mathcal{H}.
\end{aligned}
\end{equation}

\section{Two strategies in composite sampling}
\label{app:composite-sampling}

\paragraph{\textbf{Linear-decay (\cref{alg:composite_sampling-ld}).}}
This strategy linearly decreases the probability of using $\pi^\src$ with a constant decay rate after each iteration of the algorithm, conversely increasing the probability of using $\pi^\targ$. 
We have two modes with \textit{linear-decay}: 
    \emph{step-wise}, where in each time step we may change $\pib$; and 
    \emph{trajectory-wise}, where $\pib$ only changes at the start of a trajectory.
The mode is decided before executing an episode and smoothly switches from the complete \emph{step-wise} to the complete \emph{trajectory-wise} over the training process. 
We linearly decrease the probability of executing the \emph{step-wise} and use the \textit{guide} with a constant decay rate after each iteration of the algorithm, conversely increasing the probability of executing the \emph{trajectory-wise} and using the student policy. So, we initialize the probabilities $P_{\pi} = 1$ to determine $\pib$, and $P_{wise} = 1$ to determine the mode at the beginning (line~\ref{line:initialize-p}). We linearly decrease them with a constant decay rate $\upsilon$ (lines~\ref{line:decrease-pwise} and \ref{line:decrease-ppi}), determined by the training length. At the beginning of each episode, we sample $\kappa_{wise} \sim U(0,1)$, so if $\kappa_{wise} < P_{wise}$, we will execute \emph{step-wise}, or we are in \emph{trajectory-wise} (lines~\ref{line:determine-wise-begin}-\ref{line:determine-wise-end}).
Under \emph{step-wise}, at each time step, we sample from the \textit{guide} $\pi^\src$ with probability $P_{\pi}$, and sample from the student $\pi^\targ$ with probability $1-P_{\pi}$ (lines~\ref{line:decide-pib-stepwise-begin}-\ref{line:decide-pib-stepwise-end}). Under \emph{trajectory-wise}, we only make a decision once at the beginning of the trajectory (line~\ref{line:decide-pib-trajwise}).

\paragraph{\textbf{Control-switch (\cref{alg:composite_sampling-cs}).}}
To balance between the safe exploration and the sample efficiency (the samples from the target policy are relatively more valuable), the student policy keeps sampling, i.e.,  $\pib = \pi^\targ$ at the start of a trajectory (line~\ref{line:StuSam-cs}); after we meet the first $c_{t-1}>0$, we have $\pib = \pi^\src$ until the end of the trajectory (lines~\ref{line:switch_control_begin}-\ref{line:switch_control_end}). Therefore, the guide policy serves as a \textit{rescue policy} to improve safety during sampling. In addition, we leverage two replay buffers $\mathcal{D}^\src$ and $\mathcal{D}^{\targ}$ to save the guide and student samples separately (lines~\ref{line:save_experience_begin}-\ref{line:save_experience_end}), so as to control the probability $P_{\mathcal{D}^\targ}$ to use the more on-policy samples in $\mathcal{D}^{\targ}$. Thus, we have the probability $P_{\mathcal{D}^\src} = 1-P_{\mathcal{D}^\targ}$ to sample from $\mathcal{D}^\src$.
In practice, we train the safe guide to achieve $Q_{\pi^\src}^c(s,a) \leq d, s \sim \mathcal{D}, a \sim \pi^\src(\cdot \mid s)$. From the definition of $Q_{\pi^\src}^c(s,a)$, we can basically ensure $\E_{\tau \sim \rho_{\pi^\src}} \left[\sum^{\infty}_{t=0} \gamma^t  c_t \middle| s_0 = s, a_0 = a\right] \leq d$ even starting with $c_0 >0$.

\paragraph{\textbf{Main difference}.} The key distinction between linear-decay and control-switch approaches lies in the number of off-policy interactions from the student's perspective. Linear-decay entails the collection of more samples from the guide during early episodes, whereas control-switch enables the agent to collect more on-policy samples and only occasionally relies on off-policy samples from the guide following unsafe interactions. Additionally, linear-decay necessitates predefined schedules for the behaviour policy, while control-switch is adaptive. The pursuit of novel adaptive schedules presents a promising avenue for future research.

\input{algorithms/linear-decay}
\input{algorithms/control-switch}

\newpage

\section{Ablation Study }
\label{app:ablation-study}

\begin{figure}[H]
\centering
\subfigure[Behaviour policy]{\label{at-a}
\centering
\begin{minipage}[t]{0.7\linewidth}
\centering
\includegraphics[width=0.5\textwidth]{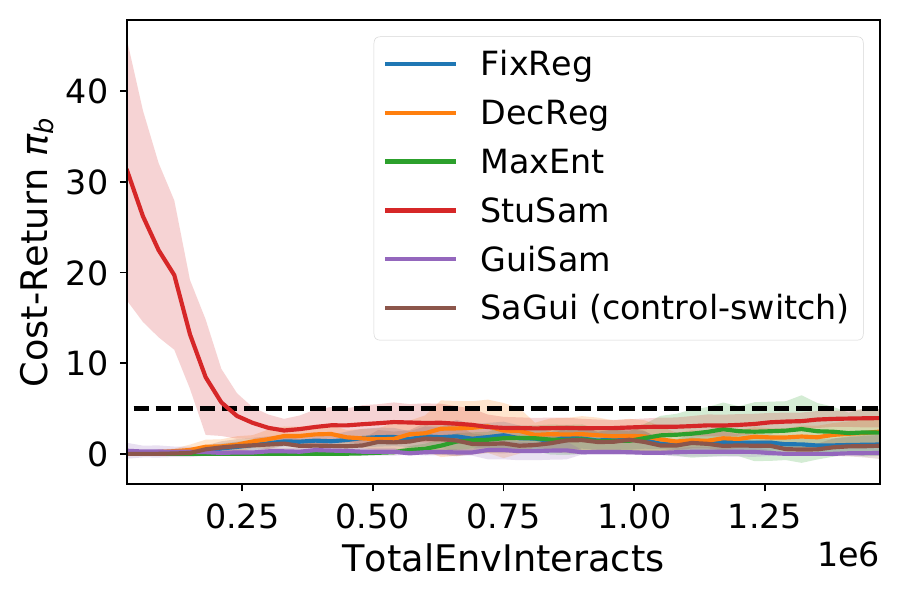}\includegraphics[width=0.5\textwidth]{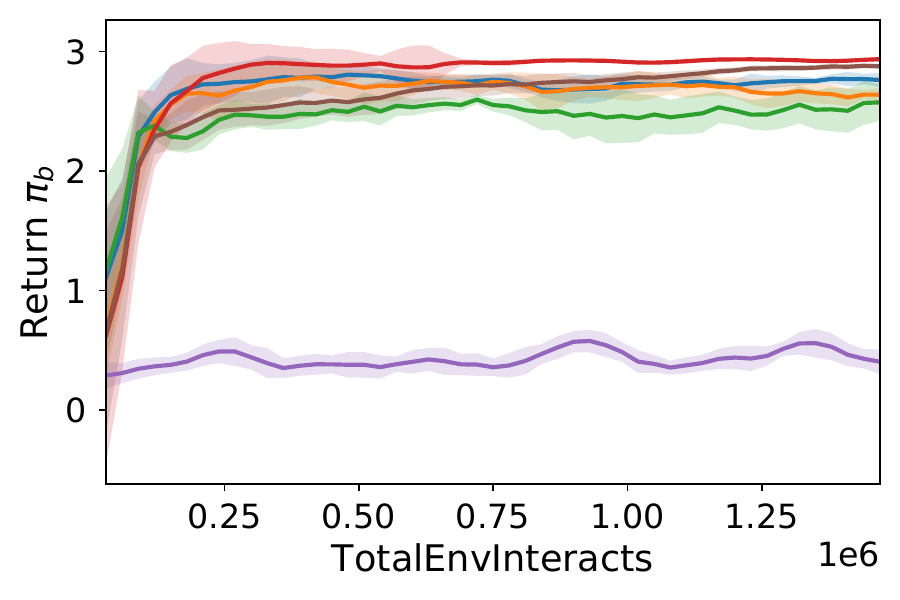}
\end{minipage}
}

\subfigure[Target policy]{\label{at-b}
\begin{minipage}[t]{0.7\linewidth}
\centering
\includegraphics[width=0.5\textwidth]{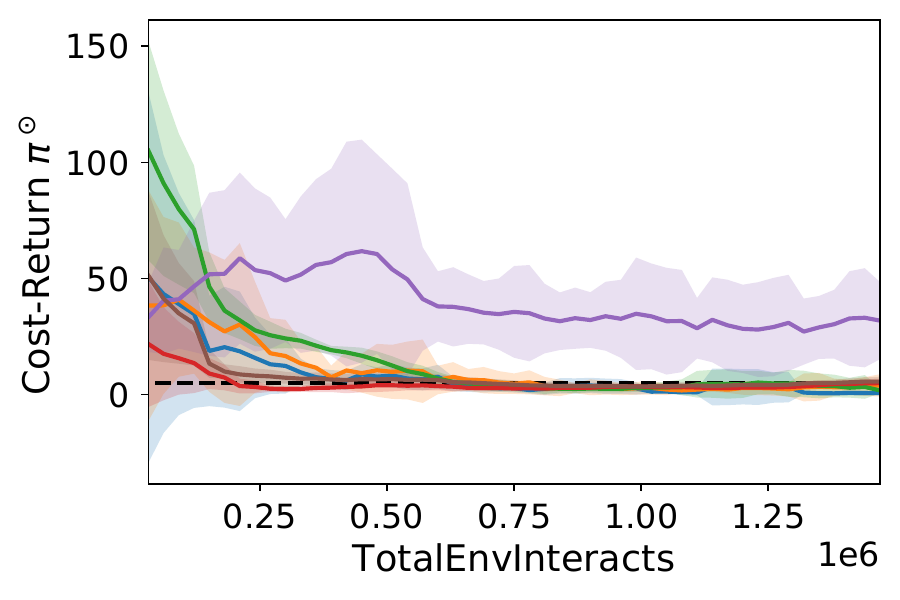}\includegraphics[width=0.5\textwidth]{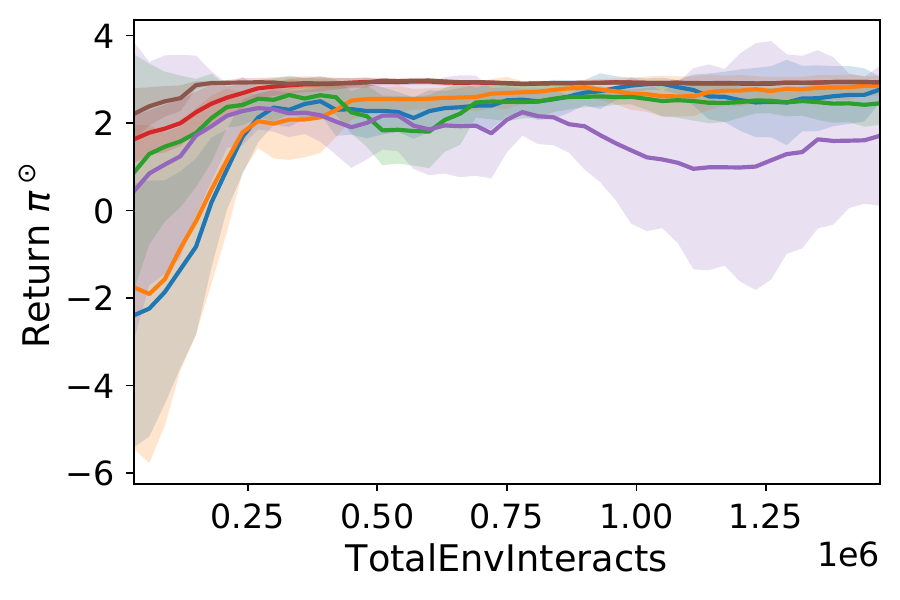}
\end{minipage}
}%
\caption{Ablation study in \textit{Static} showing the safety and performance of the behaviour policy (a) and target policy (b). The black dashed line indicates the safety threshold.}
\label{fig:AblationStudy}
\end{figure}

 \newpage 
\section{Evaluation of the target policy}
\label{app:target-policy}
\paragraph{Comparison with baselines}
In \cref{DuringTraining}, we evaluate the behaviour policy $\pib$ for all algorithms: CPO, SAC-$\lambda$, \textsc{CPO-pre}, \textsc{SAC-$\lambda$-pre}, EGPO, and \textsc{\ours{}}. So, in \cref{DuringTrainingTarget}, we show how their resulted target policy will perform during training. In all these algorithms, \textsc{\ours{}} (control-switch) is the only one that can find a safe optimal target policy in all environments. However, \textsc{\ours{}} (linear-decay) cannot achieve similar performance, especially in \textit{Semi-dynamic} and \textit{Dynamic}. We infer that \textsc{\ours{}} (linear-decay) lack samples from the target policy, especially at the early stage of training. The behaviour policy of EGPO (with benefits from the targeted expert policy) has outstanding performance during training (\cref{DuringTraining}), but EGPO fails to find a safe target policy finally. As to the pre-training baselines, \textsc{CPO-pre} and \textsc{SAC-$\lambda$-pre} do not attain obvious improvement compared to CPO and SAC-$\lambda$ that are trained from scratch. Instead, pre-training may have some negative impacts on getting a good target policy. The only exception is that \textsc{CPO-pre} is largely improved in the relatively simple environment \textit{Static}.

\begin{figure*}[h]

\subfigure[Static]{\label{DuringTrainingTarget-a}
\begin{minipage}[t]{0.32\linewidth}
\includegraphics[width=\textwidth]{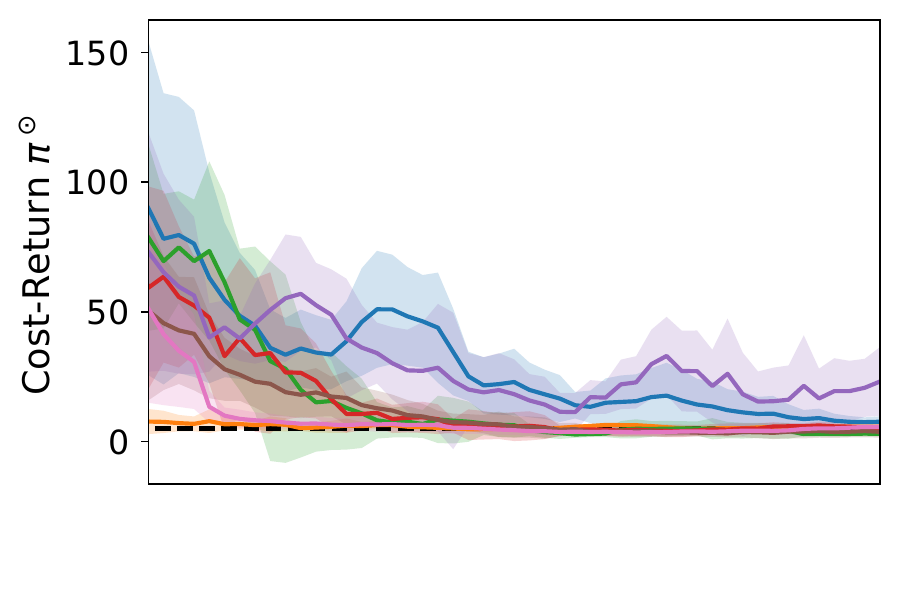}\\
\includegraphics[width=\textwidth]{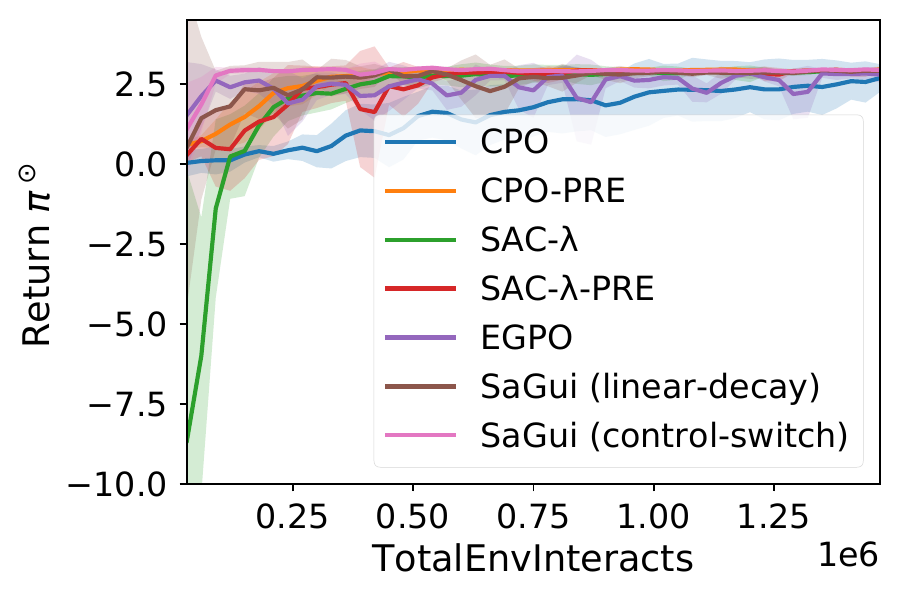}
\end{minipage}
}%
\subfigure[Semi-Dynamic]{\label{DuringTrainingTarget-b}
\begin{minipage}[t]{0.32\linewidth}
\includegraphics[width=\textwidth]{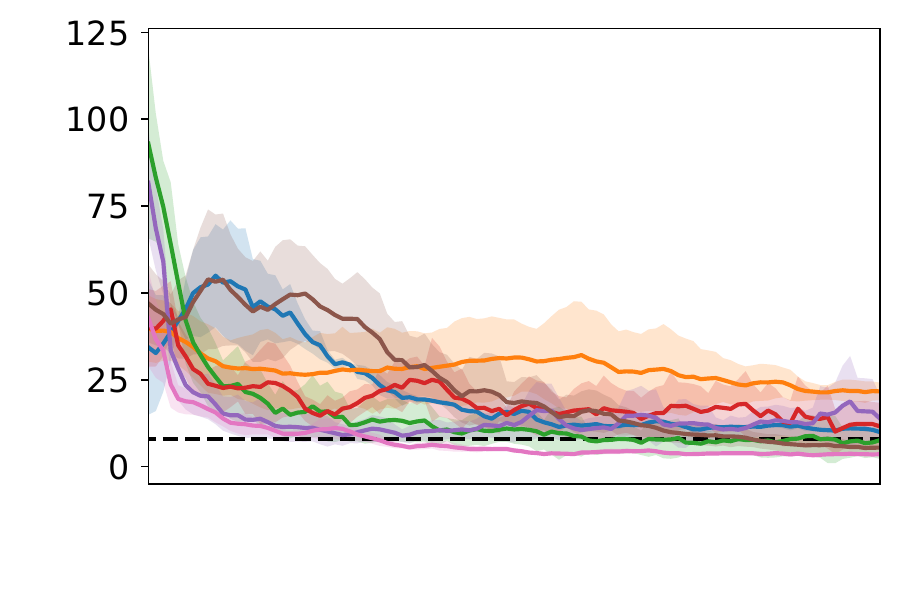}\\
\includegraphics[width=\textwidth]{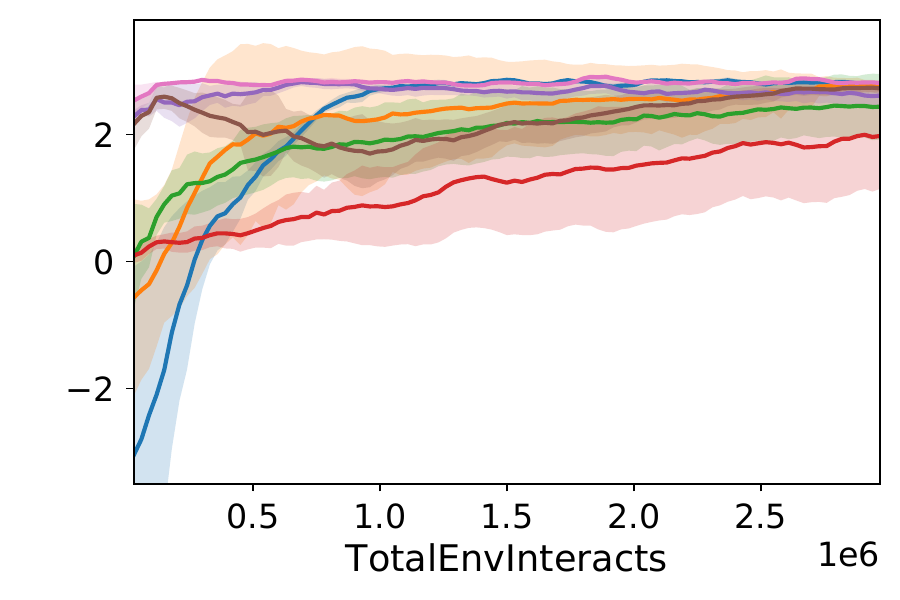}

\end{minipage}
}%
\subfigure[Dynamic]{\label{DuringTrainingTarget-c}
\begin{minipage}[t]{0.32\linewidth}
\includegraphics[width=\textwidth]{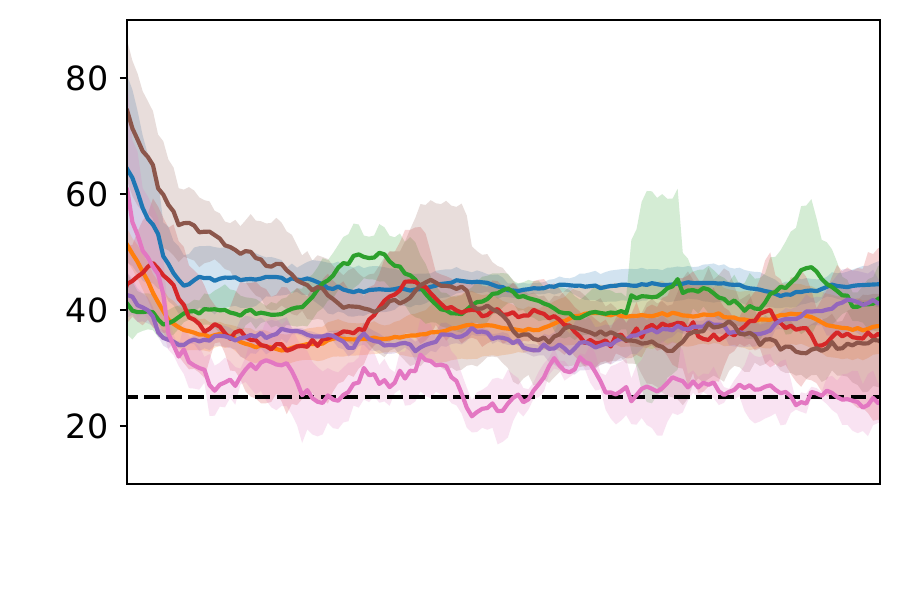}\\
\includegraphics[width=\textwidth]{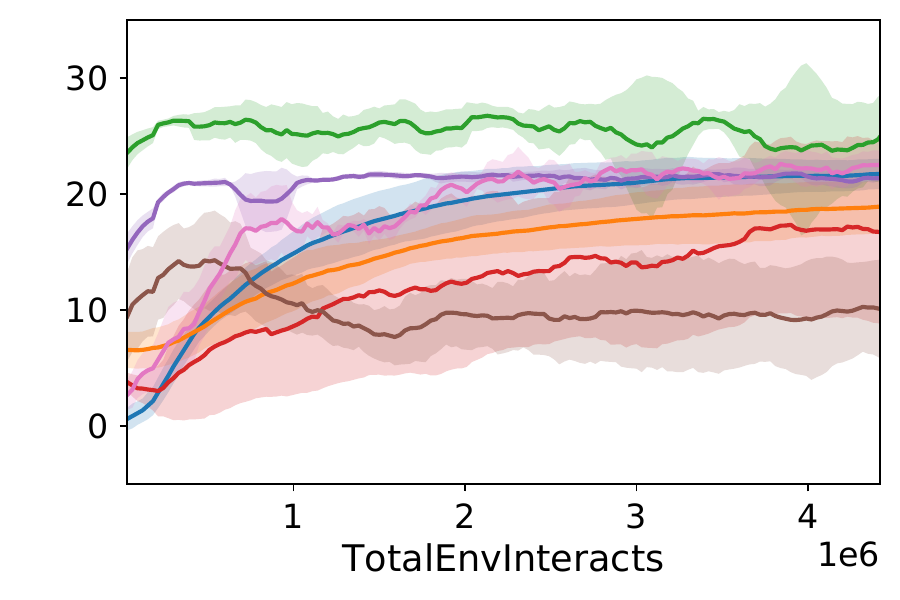}
\end{minipage}
}%
\caption{Evaluation of $\pi^\targ$ for CPO, \textsc{CPO-pre}, SAC-$\lambda$, \textsc{SAC-$\lambda$-pre}, EGPO, \textsc{\ours{}} (linear-decay), and \textsc{\ours{}} (control-switch) over ten seeds. The solid lines are the average of all runs, and the shaded area is the standard deviation. The black dashed lines indicate the safety thresholds.}
\label{DuringTrainingTarget}
\end{figure*}

\section{Hyperparameters}
\label{app:hyper}
We list the hyperparameters used in \textsc{\ours{}},
which are summarized in Table~\ref{hyper}. As to the baselines, we use the default hyperparameters in \url{https://github.com/openai/safety-starter-agents}. 
All runs in the experiment use separate feedforward Multilayer Perceptron (MLP) actor and critic networks.
The size of the neural network (all actors and critics of the algorithms) depend on the complexity of the tasks.
We use a replay buffer of size $10^6$ for each off-policy algorithm to store the experience.
The discount factor is set to be $\gamma = 0.99$, the target smoothing coefficient is set to be $0.005$ to update the target networks, and the learning rate to $0.001$.
The clipping intervale hyper-parameters $[\mathcal{I}_l, \mathcal{I}_u]$ is set to $[0.1, 2.0]$, while the sampling probabilities $P_{\mathcal{D}^\src}$ and $P_{\mathcal{D}^\targ}$ are set to $0.25$ and $0.75$, respectively.
The maximum episode length is 1000 steps in all experiments.
We set the safety constraint $d$ based on the problem.
The rest of the hyperparameters are explained in the Empirical Analysis part of the paper. 
All experiments are performed on an
Intel(R) Xeon(R) CPU@3.50GHz with 16 GB of RAM.

\begin{table}
    \centering
    \begin{tabular}{lrrrr}
         \toprule
        Parameter& Static& Semi-Dynamic& Dynamic & Note\\
        \midrule
        Size of networks& $(32,32)$& $(64,64)$ & $(256, 256)$ \\
        Size of replay buffer& $10^6$&$10^6$&$10^6$&  $|\mathcal{D}|$\\
        Batch size& 32& 64& 256 \\
        Number of epochs& $50$& $100$& $150$& \\
        Safety constraint& $5$& $8$&$25$& $d$ \\
        \bottomrule 
    \end{tabular}
    \caption{Summary of hyperparameters in \textsc{\ours{}}.}
    \label{hyper}
\end{table}

\paragraph{Safety-mapping function.}
The state spaces of the source and target task differ by the presence of the LiDAR observation of the target location.
While the source task only has a safety-related signal $x_c$, the target task has an additional goal-related signal $x_r$.
Thus, following the definition in Section~\ref{sec:problem-setting}, we can map the target state $[x_c, x_r]$ to the source state ignoring the target-related signal: $\Xi([x_c, x_r]) = [x_c]$.

\section{Expert Guided Policy Optimization}
\label{app:egpo}
We also compare our algorithms to an Expert-in-the-loop RL method called Expert Guided Policy Optimization (EGPO) that incorporates a well-performing expert policy as a demonstrator as well as a safety guardian \cite{peng2022safe}.
However, EGPO constrains safety behaviours at each timestep, which is different from our safety defined on long-term cost-return.
In terms of the safe guide, EGPO assumes the access to the well-performing expert policy, but our safe guide is task-agnostic.
Thus, the expert in EGPO depends on the target task and does not undertake the task of exploration, while our safe guide can be useful for different reward functions and enhance the exploration capabilities of the student.
Even though, EGPO can be easily adapted to our setting.
The constraint of EGPO on the guardian intervention frequency can be directly transferred to be our safety constraint.
Also, we do not minimize intervention anymore.
Once the EGPO agent starts to take unsafe actions, the expert policy can take over the control until the end.

%% file: algorithms/training_safe_explorer.tex
\begin{algorithm}[h]
\caption{Maximum exploration RL for \textit{safe guide}}
\label{alg:safe-explorer}
\textbf{Input}: $\mathcal{M}^\src$, $\alpha$, $d$\\
\textbf{Initialize}: $\mathcal{D} \leftarrow \emptyset$, $\theta^\src_\chi \text{ for } \chi \in \{\pi,R,C,\beta\}$\\
\textbf{Output}: Optimized parameters $\theta^\src_\pi$ for $\pi^\src$

\begin{algorithmic}[1] %
\FOR{each iteration} 
    \FOR{each environment step}\label{lst:env_step_begin}
        \STATE $a_t \sim \pi^\src(\cdot \mid s_t)$
        \STATE $s_{t+1} \sim \mathcal{P}(\cdot \mid s_t,a_t)$
        \STATE $r^\delta_t \leftarrow \delta(f^\ddagger(s_t),f^\ddagger(s_{t+1}))$ \COMMENT{Auxiliary task \eqref{eq:DefAuxRew}}
        \STATE $c^\src_t \leftarrow c^\src(s_t,a_t)$
        \STATE $\mathcal{D}\leftarrow  \mathcal{D} \cup \{(s_t,a_t,r^\delta_t,c^\src_t,s_{t+1}) \}$
        \COMMENT{Replay buffer}
        \ENDFOR\label{lst:env_step_end}
        \FOR{each gradient step}\label{lst:update_begin}
        \STATE Sample experience from replay buffer $\mathcal{D}$
        \FOR{$\chi \in \{\pi,R,C,\beta\}$}
        \STATE $\theta^\src_\chi \leftarrow \theta^\src_\chi - \eta_\chi \hat{\nabla}_{\theta^\src_\chi} J_{\chi}(\theta^\src_\chi) $ \COMMENT{Parameter updating}
        \ENDFOR
    \ENDFOR{\label{lst:update_end}}
\ENDFOR \label{lastline}

\end{algorithmic}
\end{algorithm}

%% file: algorithms/linear-decay.tex
\begin{algorithm}[tbp]
\caption{Composite sampling (linear-decay)}
\label{alg:composite_sampling-ld}
\textbf{Input}: $\pi^\src$, $\pi^\targ$, $\upsilon$\\
\textbf{Initialize}: $P_{\pi}\leftarrow 1$, $P_{\textrm{wise}}\leftarrow 1$ \label{line:initialize-p}\\
\textbf{Output}: $\pib$
\begin{algorithmic}[1] %
\FOR{each iteration}
\STATE $P_b(\src) = P_{\pi} $ \COMMENT{The probability of using $\pi^\src$}
\STATE $P_b(\targ)= 1-P_{\pi}$ \COMMENT{The probability of using $\pi^\targ$}
\STATE Sample $\kappa_{\textrm{wise}} \sim U(0,1)$ \label{line:determine-wise-begin}
\IF{$\kappa_{\textrm{wise}}<P_{\textrm{wise}}$}
\STATE \emph{step-wise} $\leftarrow \mathit{true}$
\ELSE
\STATE \emph{step-wise} $\leftarrow \mathit{false}$
\STATE $b \sim P_b$ \COMMENT{Choose behaviour policy} \label{line:decide-pib-trajwise}
\ENDIF \label{line:determine-wise-end}
\STATE $P_{\textrm{wise}}=P_{\textrm{wise}}-\upsilon$ \COMMENT{Decrease the probability of step-wise} \label{line:decrease-pwise}
        \FOR{each environment step} 
        \IF{\emph{step-wise}} \label{line:decide-pib-stepwise-begin}
        \STATE $b \sim P_b$ \COMMENT{Choose behaviour policy} \label{line:decide-pib-stepwise-end}
        \ENDIF 
        \ENDFOR 
        \STATE $P_{\pi}=P_{\pi}-\upsilon$ \label{line:decrease-ppi} \COMMENT{Decrease the probability of using $\pi^\src$}
        \ENDFOR
\end{algorithmic}

\end{algorithm}

%% file: algorithms/control-switch.tex
\begin{algorithm}[tbp]
\caption{Composite sampling (control-switch)}
\label{alg:composite_sampling-cs}
\textbf{Input}: $\pi^\src$, $\pi^\targ$\\
\textbf{Initialize}: $\mathcal{D}^\src \leftarrow \emptyset$, $\mathcal{D}^\targ \leftarrow \emptyset$\\
\textbf{Output}: $\pib$
\begin{algorithmic}[1]
\FOR{each iteration}
\STATE $b \leftarrow \targ$ \COMMENT{Start sampling from the student} \label{line:StuSam-cs}
\STATE \emph{control-switch}$(t)$ $\leftarrow \mathit{false}$
        \FOR{each environment step} 
            \STATE $a_t \sim \pib(\cdot \mid s_t)$
            \STATE $ E \leftarrow (s_t,a_t,r^\targ_t,r^\src_t,c_t,\mathcal{I}_t,s_{t+1})$ \COMMENT{Generate experience}
            \IF{$b = \src$} \label{line:save_experience_begin}
            \STATE $\mathcal{D}^\src\leftarrow  \mathcal{D}^\src \cup \{E \}$ \COMMENT{Save the guide samples}
            \ELSE
            \STATE $\mathcal{D}^\targ\leftarrow  \mathcal{D}^\targ \cup \{E \}$ \COMMENT{Save the student samples}
            \ENDIF \label{line:save_experience_end}
            \IF{$\neg$ \emph{control-switch}$(t)$ $\wedge$ $c_t > 0$} \label{line:switch_control_begin}
            \STATE $b \leftarrow \src$ \COMMENT{Switch behaviour policy}
            \STATE \emph{control-switch}$(t)$ $\leftarrow \mathit{true}$ 
                \label{line:switch_control_end}
            \ENDIF 
        \ENDFOR
        \ENDFOR
\end{algorithmic}
\end{algorithm}